\documentclass[11pt]{article}
\usepackage{amsthm,amsmath} 
\usepackage[colorlinks,linkcolor=red,anchorcolor=blue,citecolor=blue]{hyperref} 
\usepackage[capitalize]{cleveref} 
\usepackage{smile}
\usepackage{extarrows}
\usepackage[OT1]{fontenc}
\usepackage{fullpage}
\usepackage[protrusion=false,expansion=true]{microtype}
\usepackage{graphicx}
\usepackage{comment}
\usepackage{stmaryrd}
\usepackage[dvipsnames]{xcolor}

\usepackage{algorithm}
\usepackage{algorithmic}

\RequirePackage{times}
\RequirePackage{natbib}
 
\definecolor{mypink1}{rgb}{0.858, 0.188, 0.478}
\definecolor{mypink2}{RGB}{219, 48, 122}
\definecolor{mypink3}{cmyk}{0, 0.7808, 0.4429, 0.1412}
\definecolor{mygray}{gray}{0.6}

\newcommand{\KL}{\mathrm{KL}}
\newcommand{\supp}{\text{supp}}

\newcommand{\E}{\mathbb E}

\newcommand{\R}{\mathbb{R}}

\DeclareFontFamily{U}{mathx}{\hyphenchar\font45}
\DeclareFontShape{U}{mathx}{m}{n}{<-> mathx10}{}
\DeclareSymbolFont{mathx}{U}{mathx}{m}{n}
\DeclareMathAccent{\widebar}{0}{mathx}{"73}

\usepackage[textsize=tiny]{todonotes}
\makeatletter
\renewcommand{\todo}[2][]{\@todo[#1]{#2}}
\makeatother

\def\red#1{\textcolor{red}{#1}}

\newcommand{\normmm}[1]{{\left\vert\kern-0.25ex\left\vert\kern-0.25ex\left\vert #1
		\right\vert\kern-0.25ex\right\vert\kern-0.25ex\right\vert}}

\begin{document}

\title{ \bf Online Sparse Reinforcement Learning}
\author
{
	Botao Hao\thanks{Deepmind. E-mail: haobotao000@gmail.com.},~
	Tor Lattimore\thanks{Deepmind. E-mail:  lattimore@google.com.}
	,~
	Csaba Szepesv\'ari\thanks{Deepmind and University of Alberta. E-mail: szepi@google.com.}
	,~
	Mengdi Wang\thanks{Princeton University. E-mail: mengdiw@princeton.edu.}
}
\date{}
\maketitle

\begin{abstract}
We investigate the hardness of online reinforcement learning in  fixed horizon, sparse linear Markov decision process (MDP), with a special
focus on the high-dimensional regime where the ambient dimension is larger than the number of episodes.
Our contribution is two-fold. First, we provide a lower bound showing that linear regret is generally unavoidable in this case, even if there exists a policy that collects well-conditioned data. 
The lower bound construction uses an MDP with a fixed number of states while the number of actions scales with the ambient dimension. 
Note that when the horizon is fixed to one, the case of linear stochastic bandits, the linear regret can be avoided.
Second, we show that if the learner has oracle access to a policy that collects well-conditioned data 
then a variant of Lasso fitted Q-iteration enjoys a nearly dimension free regret
of $\tilde{O}( s^{2/3} N^{2/3})$ where $N$ is the number of episodes and $s$ is the sparsity level.
This shows that in the large-action setting, the difficulty of learning can be attributed to the difficulty of finding a good
exploratory policy.
\end{abstract}

\section{Introduction}\label{sec:intro}

Sparse models in classical statistics often yield the best of both worlds: high representation power is achieved by including many features while sparsity leads to efficient estimation. 
There is a growing interest in applying the tools developed by statisticians to sequential settings such as contextual bandits and reinforcement learning (RL). As we now explore, in
online RL this leads to a number of delicate trade-offs between assumptions and sample complexity.
The use of sparsity in reinforcement learning (RL) has been explored before in the context of policy evaluation or policy optimization 
in the batch setting \citep{kolter2009regularization, geist2011, hoffman2011regularized, painter2012greedy, ghavamzadeh2011finite, hao2020sparse}. 
As far as we know, there has been very little work on the role of sparsity in online RL.
In batch RL, the dataset is given a priori and the focus is typically on evaluating a given target policy or learning a near-optimal policy. By contrast, the central question in online RL is 
how to sequentially interact with the environment to balance the trade-off between exploration and exploitation, measured here by the cumulative regret. We ask the following question:
\begin{center}
\emph{Under what circumstances does sparsity help when minimizing regret in online RL?}
\end{center}
In sparse linear regression, the optimal estimation error rate generally scales polynomially with the sparsity $s$ and only logarithmically in the ambient dimension $d$ \citep{Wa19}. 
This is guaranteed under the sufficient and almost necessary condition that the data covariance matrix is well-conditioned, usually referred to restricted eigenvalue 
condition \citep{bickel2009simultaneous} or compatibility condition \citep{van2009conditions}.

The `almost necessary' nature of the conditions for efficient estimation with sparsity leads to an unpleasant situation when minimizing regret.
Even in sparse linear bandits, the worst-case regret is known to depend polynomially on the ambient dimension \citep[\S24.3]{lattimore2020bandit}.
The reason is simple. By definition, a learner with small regret must play mostly the optimal action, which automatically leads to poorly conditioned data.
Hence, making the right assumptions is essential in the high-dimensional regime where $d$ is large relative to the time horizon.
A number of authors have considered the contextual setting, where the regret can be made dimension free by making judicious assumptions on the context distribution
\citep{bastani2020online, wang2018minimax, kim2019doubly,ren2020dynamic, wang2020nearly}.  

When lifting assumptions from the bandit literature to RL it is essential to ensure that \textit{(a)} the assumptions still
help and \textit{(b)} the assumptions remain reasonable. In some sense, our lower bound shows that a typical assumption that helps in linear bandits is by itself insufficient in RL.
Specifically, in linear bandits the existence of a policy that collects well-conditioned data is sufficient for dimension free regret. In RL this is no longer true because finding this policy
may not be possible without first learning the transition structure, which cannot be done efficiently without well-conditioned data, which yields an irresolvable chicken-and-egg problem.

\paragraph{Contribution} 
We study online RL in episodic linear MDPs with ambient dimension $d$, sparsity $s$, episode length $H$ and number of episodes $N$.  Our contribution is two-fold:
\begin{itemize}
    \item Our first result is a lower bound showing that $\Omega(Hd)$ regret is unavoidable in the worse-case when the dimension is large, even if the MDP transition kernel can be exactly represented by a sparse linear model and there exists an exploratory policy that collects well-conditioned data. The technical contribution is to craft a new class of hard-to-learn episodic MDPs. To overcome 
    the difficulties caused by deterministic transitions from the constructed MDPs, we develop a novel stopping-time argument when calculating the KL-divergence.
\item Our second result shows that if the learner has oracle access to an exploratory policy that collects well-conditioned data, then online Lasso fitted-Q-iteration in combination with the 
explore-then-commit template achieves a regret upper bound of $\tilde{O}(H^{4/3}s^{2/3}N^{2/3})$. The proof requires a non-trivial extension of high-dimensional statistics to Markov dependent data. 
As far as we know, this is the first regret bound that has no polynomial dependency on the feature dimension $d$ in online RL.
\end{itemize}

\subsection{Related work}

Regret guarantees for online RL have received considerable attention in recent years. In episodic tabular MDPs with a homogeneous transition kernel, \cite{azar2017minimax} proved a minimax 
optimal regret of $O(\sqrt{H^2|\cS||\cA|N})$
achieved by a 
model-based algorithm. \cite{jin2018q} showed an $O(\sqrt{H^4|\cS||\cA|N})$ regret bound for Q-learning with inhomogeneous transition kernel. Under a linear MDP assumption, \cite{jin2019provably} showed an $O(\sqrt{d^3H^4N})$ regret bound for an optimistic version of least-squares value iteration. Under a linear kernel MDP assumption \citep{zhou2020provably}, \cite{yang2019reinforcement} obtained an $O(dH^{5/2}\sqrt{N})$ regret bound by a model-based algorithm while \cite{cai2019provably} obtained an $O(dH^2\sqrt{N})$ regret bound using an optimistic version of least-squares policy iteration. \cite{zanette2020frequentist} derived an $O(d^2H^{5/2}\sqrt{N})$ regret bound for randomized least-squares value iteration. None of these works considered sparsity, and consequentially the aforementioned regret bounds all have polynomial dependency on $d$. 

 \cite{JiKrAgLaSch17} and \cite{sun2019model} design algorithms for learning in RL problems with low Bellman/Witness rank, which includes sparse linear RL as a special case and obtain $O(\text{poly}(s, A, H, \log(d)))$ sample complexity where $A$ is the number of actions. More recently, FLAMBE \citep{AgKaKrSu20} achieves $O(\text{poly}(s, A, H, \log(d)))$ sample complexity in a low-rank MDP setting. It is worth mentioning that although the above results have no polynomial dependency on $d$, the sample complexity unavoidably involves polynomial dependency on the number of actions.

There are several previous works focusing on sparse linear/contextual bandits that can be viewed as a simplified online RL problem. \cite{abbasi2012online} proposed an 
online-to-confidence-set conversion approach that achieves an $O(\sqrt{sdN})$ regret upper bound, where $s$ is a known upper bound on the sparsity. The algorithm is not computationally-efficient, a deficit that
is widely believed to be unavoidable. A matching lower bound is also known, which means polynomial dependence on $d$ is generally unavoidable without additional assumptions \citep[\S24.3]{lattimore2020bandit}. More recently, under the condition that the feature vectors admit a well-conditioned exploration distribution, \cite{hao2020high} proved a dimension-free $\Omega(s^{1/3}N^{2/3})$ regret lower bound in the high-dimensional regime that can be matched by an explore-then-commit algorithm.
In the contextual setting, where the action set changes from round to round, several works 
imposed various of careful assumptions on the context distribution such that polynomial dependency on $d$ can be removed \citep{bastani2020online, wang2018minimax, kim2019doubly,ren2020dynamic, wang2020nearly}.
As far as we can tell, however, these assumptions are not easily extended to the MDP setting, where the contextual information available to the learner is not independent and identically distributed.

The use of feature selection in offline RL has also been investigated in a number of prior works. \cite{kolter2009regularization, geist2011, hoffman2011regularized, painter2012greedy, liu2012regularized} studied on-policy/off-policy evaluation with $\ell_1$-regularization for temporal-difference (TD) learning.  \cite{ghavamzadeh2011finite} and \cite{geist2012dantzig} proposed Lasso-TD to estimate the value function in Markov reward processes and derived finite-sample statistical analysis. However, the aforementioned results can not be extended to online setting directly. \cite{hao2020sparse} provided nearly optimal statistical analysis for sparse off-policy evaluation/optimization. One exception by \citet{ibrahimi2012efficient}, who 
derived an $O(p\sqrt{N})$ regret bound in high-dimensional sparse linear quadratic systems where $p$ is the dimension of the state space.

\section{Preliminaries}\label{sec:prelim}
\textbf{Notation.}
Denote by $\sigma_{\min}(X)$ and $\sigma_{\max}(X)$ the smallest and largest eigenvalues of a symmetric matrix $X$. Let $[n]=\{1,2,\ldots, n\}$. 
The relations $\lesssim$ and $\gtrsim$ stand for ``approximately less/greater than" and are used to omit constant and poly-logarithmic factors. We use $\tilde{O}(\cdot)$ to omit polylog factors. 
For a finite set $\cS$, let $\Delta_{\cS}$ be the set of probability distributions over $\cS$. 

\subsection{Problem definition}  
\paragraph{Episodic MDP.} A finite episodic Markov decision process (MDP) is a tuple $(\cX, \cA, H, P, r)$ with $\cX$ the state-space, $\cA$ the action space, $H$ the episode length, $P : \cX \times \cA \to \Delta_\cX$ the transition kernel and $r : \cX \times \cA \to [0,1]$ the reward function. As is standard, we assume that $\cX$ and $\cA$ are finite and that the reward function is known.
We write $P(x' | x, a)$ for the probability of transitioning to state $x'$ when taking action $a$ in state $x$.
A learner interacts with an episodic MDP as follows. In each episode, an initial state $x_1$ is sampled from an initial distribution $\xi_0\in\Delta_{\cX}$. Then, in each step $h \in [H]$, the 
learner observes a state $x_h \in \cX$, takes an action $a_h \in \cA$, and receives a deterministic reward $r(x_h, a_h)$. Then, the system evolves to a random next state $x_{h+1}$ according 
to distribution $P(\cdot | x_h,a_h)$. The episode terminates when $x_{H+1}$ is reached.

We define a (stationary) policy as a
function $\pi:\cX\to \Delta_{\cA}$, that maps states to distributions over actions. A nonstationary policy is a sequence of maps from histories to probability distributions over actions.  For each $h\in [H]$ and policy $\pi$, the value function $V_h^{\pi}: \cX \to \mathbb{R} $ is defined as the expected value of cumulative  rewards received under policy $\pi$ when starting from an arbitrary state at $h$th step; that is,
\begin{equation*}
 V^{\pi}_{h}(x) := \E^{\pi}\left[\sum_{h' = h}^H r(x_{h'}, a_{h'}))  \bigg | x_h = x\right]\,,
\end{equation*}
where $a_{h'}\sim \pi(\cdot|x_{h'}), x_{h'+1} \sim P(\cdot|x_{h'}, a_{h'})$, and $\mathbb E^{\pi}$ denotes the expectation over the sample path generated under policy $\pi$.
Accordingly, we also define the action-value function  $Q^{\pi}_{h}:\cX \times \cA \to \mathbb{R}$ which gives the expected cumulative  reward when the learner starts  from an  arbitrary state-action pair at the $h$th step and follows  policy $\pi$ afterwards:
\begin{equation*}
\begin{split}
    &Q^{\pi}_{h}(x,a):= r(x, a) +  \E^{\pi}  \left[\sum_{h' = h+1}^H r(x_{h'},a_{h'}) \bigg  | x_h = x, a_h = a \right]\,.
\end{split}
\end{equation*}
Note, the conditioning in the above definitions is not quite innocent. In this form the value function is not well defined for states $x$ that are not reachable by
a given policy. This is easily rectified by defining the value function in terms of the Bellman equation or by being more rigorous about the probability space. 
The above definitions are standard in the literature and are left as is for reader's convenience.

\paragraph{Bellman equation.} Since the action space and episode length are both finite, there always exists 
 an optimal policy $\pi^*$
 which gives the optimal value $V^*_h(x) = \sup_{\pi} V_h^\pi(x)$ for all $x\in \cX$ and $h\in [H]$ \citep{puterman2014Markov, sze10}. We denote the Bellman operator as
 \begin{equation*}
      [\cT V](x, a):= r(x,a)+\E_{x' \sim P(\cdot|x, a)}[V(x')]\,,
 \end{equation*}
and the Bellman equation for policy $\pi$ becomes 
\begin{equation}\label{eqn:bellman_operator}
\begin{split}
   & Q^{\pi}_h(x, a) = [\cT V_{h+1}^{\pi}](x, a)\,,\\
   &V_h^{\pi}(x) = \mathbb E_{a\sim \pi(\cdot|x)}[Q_h^{\pi}(x, a)],
     V_{H+1}^{\pi}(x) = 0\,,  
\end{split}
\end{equation}
which holds for all  $(x,a) \in \cX \times \cA $. Similarly, the Bellman optimality equation is 
\begin{equation}\label{eqn:bellman_policy}
    \begin{split}
         &Q^*_h(x, a) = [\cT V^*_{h+1}](x, a)\,,\\
         & V^*_h(x) = \max_{a\in\cA}Q^*_h(x, a),  V^*_{H+1}(x) = 0\,. 
    \end{split}
\end{equation}

\paragraph{Cumulative regret.} In the online setting, the learner aims to minimize the cumulative regret by interacting with the environment over a number of episodes. At the beginning of the $n$th episode, an initial  state $x^n_1$ is sampled from $\xi_0$ and the agent executes policy $\pi_n$. 
We measure the performance of the learner over $N$ episodes by the cumulative regret:
\begin{equation}\label{def:regret}
R_N= \sum_{n=1}^N \left(V^*_1 (x_1^n) - V^{\pi_n}_1 (x_1^n)\right)\,.
\end{equation}
The cumulative regret measures the expected loss of following the policy produced by the learner instead of the optimal policy. Therefore, the learner aims to follow a sequence of policies $\pi_1,\ldots, \pi_N$ such that the cumulative regret is
minimized.

\subsection{Sparse linear MDPs} 
Before we introduce sparse linear MDPs, we need to settle on a definition of a linear MDP. 
Let $\phi: \cX \times \cA \to \RR^d$ be a feature map which assigns to each state-action pair a $d$-dimensional feature vector. 
A feature map combined with a parameter vector $w\in \RR^d$ gives rise to the linear function $g_w:\cX\times \cA \to \RR$ defined by $g_w(x,a) = \phi(x,a)^\top w$
and the subspace  $\cG_\phi = \{ g_w \,:\, w\in \RR^d \}\subseteq \RR^{\cX \times \cA}$. 
Given a policy $\pi$ and function $f : \cX \times \cA \to \R$, let $\tilde \cT_\pi f : \cX \times \cA \to \R$ be the function defined by
\begin{align*}
[\tilde{\cT}_\pi f](x,a) = r(x,a) + \mathbb E_{x'\sim P(\cdot|x,a), a\sim \pi(a|x')}[f(x', a)]\,.
\end{align*}
We call an MDP \emph{linear} if $\cG_\phi$ is closed under $\tilde \cT_\pi$ for all policies $\pi$.\footnote{A different definition is called linear kernel MDP that the MDP transition kernel can be parameterized by a small number of parameters \citep{yang2019reinforcement, cai2019provably, zanette2020frequentist, zhou2020provably}}  
\cite{yang2019sample} and \cite{jin2019provably} have shown that this is equivalent to assuming 
$$
P(x'|x,a) = \sum_{k\in [d]} \phi_k(x,a) \psi_k(x')\,,
$$
for some functions $\psi_1,\ldots,\psi_d:\cX\to\mathbb R$ and all pairs of $(x,a)$.
Note, the feature map $\phi$ is always assumed to be known to the learner.
As far as we know, this notion of linearity was introduced by \citet{BeKaKo63,SchSe85}, who were motivated by the problem of efficiently computing the optimal policy for a known MDP with a large
state-space.

When little priori information is available on how to choose the features, agnostic choices often lead to dimensions which can be as large as the number of episodes $N$. 
Without further assumptions, no procedure can achieve nontrivial performance guarantees, even when just considering simple prediction problems (e.g., predicting immediate rewards). 
However, effective learning with many more features than the sample-size is possible when only $s\ll d$ features are relevant.
This motivates our definition of a sparse linear MDP.

\begin{definition}[{\bf Sparse linear MDP}]	\label{def:sparse_MDP}
Fix a feature map  $\phi:\cX \times \cA \to \RR^d$ and assume the episodic MDP $\cM$ is linear in $\phi$.
We say $\cM$ is $(s,\phi)$-sparse if there exists an active set $\cK\subseteq [d]$ with $|\cK|\le s$ and some functions $\psi(\cdot) = (\psi_k(\cdot))_{k\in\cK}$ such that for all pairs of $(x,a)$:
$$P(x' |x,a) = \sum_{k\in \cK} \phi_k(x,a) \psi_k(x')\,.$$
\end{definition}

\section{Hardness of online sparse RL}\label{sec:lower_bound}
In this section we illustrate the fundamental hardness of online sparse RL in the high-dimensional regime by proving a minimax regret lower bound. 
The high-dimensional regime is  referred to $N\leq d$. We first introduce a notion of an exploratory policy.

\begin{definition}[{\bf Exploratory policy}]
\label{def:good_exploratory}
Let $\Sigma^{\pi}$ be the expected uncentered covariance matrix induced by policy $\pi$ and feature map $\phi$, given by
\begin{equation}\label{eqn:expected_cov}
    \Sigma^{\pi}:=\mathbb E^{\pi}\left[\frac{1}{H}\sum_{h=1}^{H}\phi(x_{h},a_{h})\phi(x_{h},a_{h})^{\top}\right] \,,
\end{equation} 
where $x_1\sim \xi_0, a_h\sim \pi(\cdot|x_h), x_{h+1}\sim P(\cdot|x_h, a_h)$ and $\mathbb E^{\pi}$ denotes expectation over the sample path generated under policy $\pi$. We call a policy $\pi$  \emph{exploratory} if $  \sigma_{\min}(\Sigma^{\pi})>0$.
\end{definition}

\begin{remark}
Intuitively, $\sigma_{\min}(\Sigma^{\pi})$ characterizes how well the policy $\pi$ explores in the feature space. Similar quantities also appear in the assumptions in the literature 
to ensure the success of policy evaluation/optimization with linear function approximation \citep[assumption A.4]{abbasi2019politex}, \citep[theorem 2]{duan2020minimax}, \citep[assumption A.3]{lazic2020maximum},
\citep[assumption A.3]{abbasi2019exploration} and \citep[assumption 6.2]{agarwal2020theory}.
\end{remark}
\begin{remark}
In the tabular case, we can choose $\phi(x, a)$ as a basis vector in $\R^{|\cX| \times |\cA|}$. Let $\mu^\pi(x,a)$ be the frequency of visitation for state-action pair $(x,a)$ under policy $\pi$ and initial distribution $\xi_0$: 
\begin{equation*}
    \mu^{\pi}(x,a) = \frac{1}{H}\sum_{h=1}^H\mathbb E^{\pi}\big[\mathbb I((x_h, a_h)=(x, a))\big].
\end{equation*}
Then $\sigma_{\min}(\Sigma^{\pi})>0$ implies $\min_{x,a} \mu^{\pi}(x,a)>0$. This means an exploratory policy in the tabular case will have positive visitation probability of each state-action pair.
\end{remark}

The next theorem is a kind of minimax lower bound for online sparse RL. 
The key steps of the proof follow, with details and technical lemmas deferred to the appendix.
\begin{theorem}[{\bf Minimax lower bound in high-dimensional regime}]\label{thm:lower_bound}
For any algorithm $\pi$, there exists a sparse linear MDP $\cM$ and associated exploratory policy $\pi_e$ for which $\sigma_{\min}(\Sigma^{\pi_e})$ is a strictly positive universal constant 
independent of $N$ and $d$, such that for any $N\leq d$,
\begin{equation*}
   R_N \geq \frac{1}{128} Hd\,.
\end{equation*}
\end{theorem}
This theorem states that even if the MDP transition kernel can be exactly represented by a sparse linear model and there exists an exploratory policy, 
the learner could still suffer linear regret in the high-dimensional regime. This is in stark contrast to linear bandits, where the existence of an exploratory policy is sufficient for
dimension-free regret. The problem in RL is that \emph{finding} the exploratory policy can be very hard.

\begin{proof}[Proof of Theorem~\ref{thm:lower_bound}]
The proof uses the standard information theoretic machinery, but with a novel hard-to-learn MDP construction and KL divergence calculation based on a stopping time argument. The intuition is to construct an informative state with only one of a large set of actions leading to the informative state deterministically. And the exploratory policy has to visit that informative state to produce well-conditioned data. In order to find this informative state, the learner should take a large number of trials that will suffer high regret.

\begin{figure}[h]\label{fig:MDP}
\centering
\includegraphics[scale=0.6]{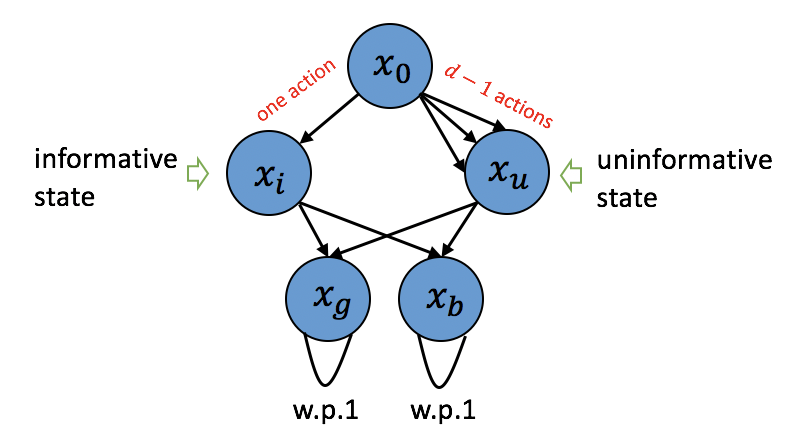}
\caption{A hard-to-learn  MDP instance that includes an informative state and an uninformative state.}
\end{figure}
\paragraph{Step 1: Construct a set of hard MDP instances.} Let the state space $\cX$ consists of $\{x_0, x_{\text{i}},x_{\text{u}}, x_{\text{g}}, x_{\text{b}}\}$. Here, $x_0$ is the initial state, $x_{\text{i}}$ and $x_{\text{u}}$ refer to informative and uninformative states, $x_{\text{g}}$ and $x_{\text{b}}$ refer to high-reward and low-reward states. Construct $d$ different hard MDP instances: $\{\cM_1,\ldots, \cM_d\}$ and they only differ at which action brings the learner to $x_{\text{i}}$. For each MDP $\cM_k, k\in[d]$,
\begin{equation}\label{eqn:theta}
    \theta = \big(\underbrace{\varepsilon, \ldots, \varepsilon}_{s-1}, 0, \ldots, 0, 1/2\big)^{\top} \,,
\end{equation}
where $\varepsilon>0$ is a small constant to be tuned later, and $\bar{\theta}^{(k)}\in\mathbb R^{2d+2}$ as
\begin{equation}\label{def:theta_k}
    \bar{\theta}^{(k)} = (\theta^{\top}, 1, 1,  \underbrace{0, \ldots, 0}_{k-1}, 1, \underbrace{0, \ldots, 0}_{d-k})^{\top}\,.
\end{equation}
We specify the transition probability of $\cM_k$ in the following steps:
\begin{enumerate}
    \item Let the initial state $x_0$ associated with $d$ actions as $\cA_1 = \{a_1^0, \ldots, a_d^0\}$.
    The transitions from $x_0$ to either $x_{\text{i}}$ or $x_{\text{u}}$ are \emph{deterministic}. In MDP $\cM_k$, only taking action $a_k^0$ brings the learner to $x_{\text{i}}$, and taking any other action \emph{except} $a_k^0$ brings the learner to $x_{\text{u}}$. This information is parameterized into the last $d$ coordinates of $ \bar{\theta}^{(k)}$. Specifically, let 
    \begin{equation*}
    \begin{split}
         &\phi(x_0, a_k^0) = (\underbrace{0,\ldots, 0}_{d+2},\underbrace{0,\ldots,0}_{k-1}, 1, \underbrace{0,\ldots, 0}_{d-k}, 1)\in\mathbb R^{2d+3},\\
         &\phi(x_0, a_j^0) = (\underbrace{0,\ldots, 0}_{d+2},\underbrace{0,\ldots,0}_{j-1}, 1, \underbrace{0,\ldots, 0}_{d-j}, 1)\in\mathbb R^{2d+3}.
    \end{split}
    \end{equation*}
for $j\in[d]$ but $j \neq k$. In addition, we let $\psi(x_{\text{i}})=(\bar{\theta}^{(k)\top}, 0)\in\mathbb R^{2d+3}$ and $\psi(x_{\text{u}})=(-\bar{\theta}^{(k)\top}, 1)\in\mathbb R^{2d+3}$. Now we can verify for $a_k^0$:
\begin{equation*}
    \begin{split}
       & \mathbb P(x_{\text{u}}|x_{0}, a_k^0) = \phi(x_0, a_k^0)^{\top}\psi(x_{\text{u}}) = 0,\\
       &  \mathbb P(x_{\text{i}}|x_{0}, a_k^0) = \phi(x_0, a_k^0)^{\top}\psi(x_{\text{i}}) = 1,
    \end{split}
\end{equation*}
and for $a_j^0$ $(j\neq k)$:
\begin{equation*}
    \begin{split}
       & \mathbb P(x_{\text{u}}|x_{0}, a_j^0) = \phi(x_0, a_j^0)^{\top}\psi(x_{\text{u}}) = 1,\\
       &  \mathbb P(x_{\text{i}}|x_{0}, a_j^0) = \phi(x_0, a_j^0)^{\top}\psi(x_{\text{i}}) = 0,
    \end{split}
\end{equation*}

    \item We construct a feature set $\cS$ associated to $x_{\text{u}}$ and a feature set $\cH$ associated to $x_{\text{i}}$:
\begin{equation*}\label{eqn:action_set}
    \begin{split}
        \cS = \Big\{z\in\mathbb R^{d}\Big|&z_d = 0,z_j\in\{-1, 0, 1\} \ \text{for} \ j\in[d-1], \|z\|_1=s-1 \Big\}\,, \\
        \cH = \Big\{z\in\mathbb R^{d}\Big|&z_j\in\{-1, 1\} \ \text{for} \ j\in[d-1],z_d = 1\Big\} \,.
    \end{split}
\end{equation*}
Let $\cA_2 = \{a_1^{\text{u}},\ldots, a^{\text{u}}_{|\cS|}\}$ be the action set associated with $x_{\text{u}}$ and  $\cA_3 = \{a_1^{\text{i}},\ldots, a^{\text{i}}_{|\cH|}\}$ be the action set associated with $x_{\text{i}}$. We write $\varphi(x_{\text{u}}, a_j^{\text{u}})$ as the $j$th element in $\cS$ and $\varphi(x_{\text{i}}, a_j^{\text{i}})$ as the $j$th element in $\cH$. Denote 
$$
\phi(x_{\text{i}}, a_j^{\text{i}}) = (\varphi(x_{\text{i}}, a_j^{\text{i}})^{\top}, \underbrace{0, \ldots, 0}_{d+2}, 1)^{\top}\in\mathbb R^{2d+3}
$$ and $\psi(x_b) = (\bar{\theta}^{(k)\top}, 0)\in\mathbb R^{2d+3}$, $\psi(x_g) = (-\bar{\theta}^{(k)\top}, 1)\in\mathbb R^{2d+3}$. At informative state $x_{\text{i}}$, the learner can take
action $a_j^{\text{i}}\in\cA_3$ and transits to either $x_{\text{g}}$ or $x_{\text{b}}$ according
to
\begin{equation*}
    \begin{split}
     &P(x_{\text{g}}|x_{\text{i}}, a_j^{\text{i}}) = \varphi(x_{\text{i}}, a_j^{\text{i}})^{\top}\theta = \phi(x_{\text{i}}, a_j^{\text{i}})^{\top}\psi(x_g), \\ &P(x_{\text{b}}|x_{\text{i}}, a_j^{\text{i}}) = 1- \varphi(x_{\text{i}}, a_j^{\text{i}})^{\top}\theta = \phi(x_{\text{i}}, a_j^{\text{i}})^{\top}\psi(x_b)\,,
\end{split}
\end{equation*}
that satisfy the sparse linear MDP assumption. We specify $\phi,\psi$ similarly when the learner at $x_{\text{u}}$.
\item At $x_{\text{g}}$  or $x_{\text{b}}$, the learner will stay the current state for the rest of current episode no matter what actions to take.
\end{enumerate}
One can verify that the above construction so far satisfies the sparse linear MDP assumption in Definition \ref{def:sparse_MDP}. In the end, the reward function is set to be $r(x, a) = 1$ if $x=x_{\text{g}}$ and $r(x,a) = 0$ otherwise. We now finish the construction of all the essential ingredients of $\{\cM_1,\ldots, \cM_d\}$. 

\begin{remark}
For $\cM_k$, the overall action set will be $\cA_1\cup\cA_2\cup \cA_3$. Now we specify the transitions that we have not mentioned so far. At $x_0$, all the actions from $\cA_2$ and $\cA_3$ bring the learner to $x_{\text{u}}$. At $x_{\text{i}}$, all the actions from $\cA_1$ and $\cA_2$ bring the learner to either $x_{\text{g}}$ or $x_{\text{b}}$ with the same probability as $a_1^{\text{i}}$. At $x_{\text{u}}$, actions from $\cA_1$ and $\cA_3$ bring the learner to either $x_{\text{g}}$ or $x_{\text{b}}$ with the same probability as $a_1^{\text{u}}$. 
\end{remark}

\paragraph{Step 2: Construct an alternative set of MDPs.} For each $k\in[d]$, the second step is to construct an alternative MDP $\tilde{\cM}_k$ that is hard to distinguish from $\cM_k$ and
for which the optimal policy for $\cM_k$ is suboptimal for $\tilde{\cM}_k$ and vice versa. 
Fix a sequence of policies $\{\pi_1, \ldots, \pi_N\}$. 
Let $\cD_n =(S_1^n, A_1^n,\ldots, S_H^n, A_H^n)$ be the sequence of state-action pairs in $n$th episode produced by $ \pi_n$. Define $\cF_h^n = \sigma(\cD_1,\ldots,\cD_{h-1}, S_{1}^n, A_{1}^n,\ldots, S_{h-1}^n, A_{h-1}^n, S_h^n)$. Let $\mathbb F = (\cF_{h}^n)_{h\in[H], n\in[N]}$ be a filtration. Define the stopping time with respect to $\mathbb F$:
\begin{equation*}
\begin{split}
\tau_k = N\wedge \min\Big\{n:A_{1}^{n} =a_k^0\Big\}
\end{split}
\end{equation*}
that is the first episode the learner reaches the informative state. In other words, for $n\leq \tau_k-1$, the learner always transits to $x_{\text{u}}$ from $x_0$. At  $x_{\text{u}}$, the learner acts similarly as facing linear bandits where the number of arms is $|\cS|$.

For $k\in[d]$, let $\mathbb P_k, \tilde{\mathbb P}_k $ be the laws of $\cD_1,\ldots,\cD_{\tau_k-1}$ induced by the interaction of $\{\pi_1, \ldots, \pi_N\}$ and $\cM_k, \tilde{\cM}_k$ accordingly. Let $\mathbb E_k, \tilde{\mathbb E}_k$ be the corresponding expectation operators. In addition, denote a set $\cS'$ as 
\begin{equation}\label{def:S_prime}
\begin{split}
     \cS' = &\Big\{z\in\mathbb R^d\Big|\|z\|_1=s-1,z_j\in\{-1, 0, 1\} \ \text{for} \ j\in\{s,s+1,\ldots,d-1\},\\
     &z_j = 0 \ \text{for} \ j=\{1,\ldots, s-1, d\}\Big\} \,.
\end{split}
\end{equation}
Then we let
\begin{equation}\label{def:x_tilde}
   \tilde{z}^{(k)} = \argmin_{z\in\cS'}\mathbb E_k \left[\sum_{n=1}^{\tau_k} \big \langle \varphi(S_2^n, A_2^n), z \big\rangle^2\right] \,,
\end{equation}
and construct the alternative $\tilde{\theta}^{(k)} = \theta + 2\varepsilon \tilde{z}^{(k)}$ where $\varepsilon$ appears in Eq.~\eqref{eqn:theta}. This is in contrast with $\bar{\theta}^{(k)}$ in Eq.~\eqref{def:theta_k} that specifies the original MDP $\cM_k$. All the other ingredients of $\tilde{\cM}_k$ are the same as $\cM_k$. Thus, we have constructed an alternative set of MDPs.

\paragraph{Step 3: Regret decomposition.}
Let $R_{N}(\cM_k)$ be the cumulative regret of a sequence of policies $\{\pi_1,\ldots, \pi_N\}$ interacting with MDP $\cM_k$ for $N$ episodes. Recall that from the definition in Eq.~\eqref{def:regret}, we have
\begin{equation*}
    R_N(\cM_k)= \sum_{n=1}^N \Big(V^{\pi^*}_1 (x_1^n) - V^{\pi_n}_1 (x_1^n)\Big)\,.
\end{equation*}
Denote $a^* = \argmax_{a_j^{\text{u}}\in\cA_2} \varphi(x_{\text{u}}, a_j^{\text{u}})^{\top}\theta$ be the optimal action when the learner is at $x_{\text{u}}$. The optimal policy $\pi^*$ of MDP $\cM_k$ behaves in the following way for each episode:
\begin{itemize}
    \item At state $x_0$, the optimal policy takes an arbitrary action except $a_k^0$ to state $x_{\text{u}}$. There is no reward collected so far.
    \item  At state $x_{\text{u}}$, the optimal policy takes action $a^*$ and transits to good state $x_{\text{g}}$ with probability $\varphi(x_{\text{u}}, a^*)^{\top}\theta$ or bad state $x_{\text{b}}$ with probability $1-\varphi(x_{\text{u}}, a^*)^{\top}\theta$.
    \item The learner stays at the current state for the rest of  current episode.
\end{itemize}
Then the value function of $\pi^*$ at $n$th episode  is  
\begin{equation*}
\begin{split}
    V^{\pi^*}_1 (x_1^n)&= (H-1)\mathbb P(A_2^n = a^*)\\
    &=(H-1)\varphi(x_{\text{u}}, a^*)^{\top}\theta = (H-1)(s-1)\varepsilon\,.
\end{split}
\end{equation*}
We decompose $R_N(\cM_k)$ according to the stopping time $\tau_k$:
\begin{equation}\label{eqn:R}
\begin{split}
    &R_N(\cM_k) \geq  \sum_{n=1}^{\tau_k-1} \Big(V^{\pi^*}_1 (x_1^n) - V^{\pi_n}_1 (x_1^n)\Big) \\
    &\geq \frac{H}{8} \mathbb E_k\Big[\tau_ks\varepsilon - \sum_{n=1}^{\tau_k-1}\langle \varphi(S_2^n, A_2^n), \theta\rangle\Big]\\
     & = \frac{H}{8} \mathbb E_k\Big[\tau_ks\varepsilon - \sum_{n=1}^{\tau_k-1}\sum_{j=1}^{s-1} \varphi_j(x_{\text{u}}, A_2^n)\varepsilon\Big]\,,
\end{split}
\end{equation}
where the last equation is due to $S_2^n$ is always $x_{\text{u}}$ until the stopping time $\tau_k$.

Define an event 
    \begin{equation*}
        \cD_k = \left\{\sum_{n=1}^{\tau_k-1} \sum_{j=1}^{s-1}\varphi_j(x_{\text{u}}, A_2^n)\leq \frac{\tau_ks}{2}\right\} \,.
\end{equation*}
The next claim shows that when $\cD_k$ occurs, the regret is large in MDP $\cM_k$, while if it does not occur, then the regret is large in MDP $\smash{\tilde \cM}_k$.
The detailed proof is deferred to
Appendix \ref{sec:claim_regret_lower}.
\begin{claim}\label{claim:regret_lower} 
Regret lower bounds with respect to event $\cD_k$:
\begin{equation*}
    \begin{split}
       R_N(\cM_k)+&R_N(\tilde{\cM}_k)\geq\frac{Hs\varepsilon}{8}\Big(\mathbb E_k[\tau_k]+\tilde{\mathbb E}_{k}[\tau_k\mathbb I(\cD_k^c)]-\mathbb E_{k}[\tau_k\mathbb I(\cD_k^c)]\Big)\,. \\
    \end{split}
\end{equation*}
\end{claim}

We construct an additional MDP $\cM_0$ such that when the learner is at $x_0$, no matter what actions to take, the learner will always transit to the uninformative state $x_{\text{u}}$. All the other structures remain the same with $\{\cM_1,\ldots, \cM_d\}$. Let $\mathbb P_0$ be the laws of $\cD_1, \ldots,\cD_{\tau_k-1}$ induced by the interaction of $\pi$ and $\cM_0$ and let $\mathbb E_0$ be the corresponding expectation operators. Then from Pinsker's inequality (Lemma \ref{lemma:pinsker} in the Appendix), for any $k\in[d]$,
\begin{equation*}
\begin{split}
     &\Big|\mathbb E_0[\tau_k] - \mathbb E_{k}[\tau_k]\Big|\leq N\sqrt{\frac{1}{2}\KL(\mathbb P_0\| \mathbb P_k)}\,, \\
     &\Big|\tilde{\mathbb E}_{k}[\tau_k\mathbb I(\cD_k^c)]-\mathbb E_{k}[\tau_k\mathbb I(\cD_k^c)]\Big|\leq N\sqrt{\frac{1}{2}\KL(\tilde{\mathbb P}_k\| \mathbb P_k)}\,,
\end{split}
\end{equation*}
where $\KL(\mathbb P, \mathbb P')$ is the KL divergence between probability measures $\mathbb P$ and $\mathbb P'$.
Combining with Claim \ref{claim:regret_lower}, we have 
\begin{equation}\label{eqn:bound1}
    \begin{split}
       & R_N(\cM_k)+R_N(\tilde{\cM}_k)\geq \frac{Hs\varepsilon}{8}\Big(\mathbb E_0[\tau_k]-d\sqrt{\frac{1}{2}\KL(\mathbb P_0\| \mathbb P_k)}-d\sqrt{\frac{1}{2}\KL(\tilde{\mathbb P}_k\| \mathbb P_k)}\Big)\,,
    \end{split}
\end{equation}
where we consider the high-dimensional regime such that $N\leq d$.

\paragraph{Step 4: Calculating the KL divergence.} We make use of the following bound on the KL divergence 
between $\tilde{\mathbb P}_k$ and $\mathbb P_k$, $\mathbb P_0$ and $\mathbb P_k$, which
formalises the intuitive notion of information. When the KL divergence is small, the algorithm is unable
to distinguish the two environments. 
The detailed proof is deferred to Appendix \ref{sec:claim_KL}. 
\begin{claim}\label{claim:KL_bound}
The KL divergences between $\tilde{\mathbb P}_{k}$ and $\mathbb P_{k}$, $\mathbb P_0$ and $\mathbb P_k$ are upper bounded by the following when $N\leq d$:
\begin{equation}\label{eqn:KL_bound}
\begin{split}
    \text{KL}(\tilde{\mathbb P}_{k}\| \mathbb P_{k})\leq 8\varepsilon^2(s-1)^2\,,\text{KL}(\mathbb P_{0}\| \mathbb P_{k}) = 0\,.
\end{split}
\end{equation}
\end{claim}
Combining with Eq.~\eqref{eqn:bound1} and summing over the set of MDPs $\{\cM_1,\ldots, \cM_d\}$,
\begin{equation*}
\begin{split}
  &\sum_{k=1}^d\Big(R_N(\cM_k) + R_N(\tilde{\cM}_k)\Big) \geq \frac{Hs\varepsilon}{8}\Big( \sum_{k=1}^d\mathbb E_0[\tau_k]-d^2\sqrt{8\varepsilon^2s^2}\Big)\,.
  \end{split}
\end{equation*}
 Picking $\varepsilon = 1/(8s)$, we have 
\begin{equation*}
\begin{split}
     &\sum_{k=1}^d\Big(R_N(\cM_k) + R_N(\tilde{\cM}_k)\Big)\geq\frac{H}{32}\Big(\sum_{k=1}^d \mathbb E_0[\tau_k]-\frac{d^2}{4}\Big)\,.
     \end{split}
\end{equation*}

\paragraph{Step 5: Summary.} 
From the definition of the stopping time, one can see $ \sum_{k=1}^d\mathbb E_0(\tau_k)\geq \sum_{k=1}^d k \geq d^2/2.$ Therefore, 
\begin{equation*}
\begin{split}
     &\sum_{k=1}^d\Big(R_N(\cM_k) + R_N(\tilde{\cM}_k)\Big)\geq\frac{1}{128}Hd^2\,.
\end{split}
\end{equation*}
Among two sets of MDPs $\{\cM_k\}_{k=1}^d$ and $\{\tilde{\cM}_k\}_{k=1}^d$, for any sequence of policies $\{\pi_1,\ldots, \pi_N\}$, there must exist a MDP $\cM_k$ such that 
\begin{equation*}
   R_N(\cM_k) \geq \frac{1}{128}Hd\,.
\end{equation*}
This finishes the proof.

\end{proof}

\section{Online Lasso Fitted-Q-iteration}\label{sec:upper_bound}
In this section we prove that if the learner has oracle access to an exploratory policy, the online Lasso fitted-Q-iteration (Lasso-FQI) algorithm can have a dimension-free $\tilde{O}(N^{2/3})$ regret upper bound. We first introduce the online Lasso-FQI. Suppose the learner has the oracle access to an exploratory policy $\pi_e$ (defined in Definition \ref{def:good_exploratory}). The algorithm uses the explore-then-commit template and includes the following three phases:
\begin{itemize}
    \item \textbf{Exploration phase.} The exploration phase includes $N_1$ episodes where $N_1$ will be chosen later based on regret bound and can be factorized as $N_1 = RH$, where $R>1$ is an integer. At the beginning of each episode, the agent receives an initial state drawn from $\xi_0$ and executes the rest steps following the exploratory policy $\pi_e$. Let the dataset collected in the exploration stage as $\cD$.
    \item \textbf{Learning phase.} Split $\cD$ into $H$ folds: $\{\cD_1,\ldots, \cD_H\}$ and each fold consists of $R$ episodes. Based on the exploratory dataset $\cD$, the agent executes an extension of fitted-Q-iteration \citep{ernst2005tree, antos2008fitted} combining with Lasso \citep{tibshirani1996regression} for feature selection. To define the algorithm, it is useful to introduce $Q_w(x,a) = \phi(x,a)^{\top}w$. For $a<b$, we also define the
 operator $\Pi_{[a,b]}: \RR \to [a,b]$ that projects its input to $[a,b]$, i.e., $\Pi_{[a,b]}(x) = \max(\min(x,b),a)$. Initialize $\hat{w}_{H+1} = 0$. At each step $h\in[H]$, we fit $\hat{w}_{h}$ through Lasso: 
\begin{equation}\label{eqn:lasso}
\begin{split}
    \hat{w}_{h}= &\argmin_{w} \ 
    \frac{1}{|\cD_h|}\sum_{(x_i,a_i,x_i')\in\cD_h}(y_i-\phi(x_i, a_i)^{\top}w)^2+\lambda_1 \|w\|_1 \,,
\end{split}
\end{equation}
where $ y_i = \Pi_{[0, H]} \max_{a\in\cA} Q_{\hat{w}_{h+1}}(x'_i, a)$ and $\lambda_1$ is a regularization parameter.
\item \textbf{Exploitation phase.} For the rest $N-N_1$ episodes, the agent commits to the greedy policy with respect to the estimated Q-value $\{Q_{\hat{w}_{h}}\}_{h=1}^H$. 
\end{itemize}
The full algorithm of online Lasso-FQI is summarized in Algorithm \ref{alg:estc}. 
\begin{remark}
A key observation of Algorithm \ref{alg:estc} is that the expected covariance matrix of data collected in the exploration phase could be well-conditioned due to the use of exploratory policy, e.g.,
\begin{equation*}
    \sigma_{\min}\Big(\mathbb E^{\pi_e}\Big[\frac{1}{N_1H}\sum_{n=1}^{N_1}\sum_{h=1}^H\phi(x_h^n, a_h^n)\phi(x_h^n, a_h^n)^{\top}\Big]\Big)>0\,.
\end{equation*}
This is the key condition to ensure the success of fast sparse feature selection in the learning and exploitation phases and eliminate the polynomial dependency of $d$ in the cumulative regret.
\end{remark}
{\small
\begin{algorithm}[htb!]
	\caption{Online Lasso-FQI}
	\begin{algorithmic}[1]\label{alg:estc}
		\STATE
		\textbf{Input:} An episodic MDP $\cM=(\cX, \cA, P, r, H)$, an exploratory policy $\pi_e$, exploration length $N_1$, regularization parameter $\lambda_1$;
		
		\red{\# \emph{exploration phase}} 
		\STATE \textbf{Initialize.} $\cD = \emptyset$.
		\FOR{$n= 1, \cdots, N_1$}
		\STATE Receive an initial state $x_1^n$.
		\FOR{$h = 1, \ldots, H$}
\STATE Take the action $a_h^{n} = \pi_e(\cdot|x_h^{n})$ and observe $x^{n}_{h+1}$. 
\STATE Let $\cD = \cD \cup \{x^{n}_{h}, a_h^{n}, x^{n}_{h+1}\}$.
\ENDFOR
\ENDFOR

\red{\# \emph{learning phase: Lasso fitted-Q-iteration}} 
\STATE Partition the
dataset $\cD$ into $H$ folds such that each fold $\cD_h$ has $R$ different episodes.
\STATE Initialize $Q_{\hat{w}_{H+1}}(x, a) = 0$.
\FOR{$h=H,\ldots, 1$}
		\STATE Calculate regression targets for each $(x_i, a_i, x_i')\in\cD_h$:
		$$
		  y_i = \Pi_{[0, H]} \max_{a\in\cA} Q_{\hat{w}_{h+1}}(x'_i, a) \,.
		$$
	\STATE Build training set $\{(x_i, a_i), y_i\}_{i\in\cD_h}$ and fit $\hat{w}_{h}$ through sparse linear regression in Eq.~\eqref{eqn:lasso}.
		\ENDFOR
		
	\red{\# \emph{exploitation phase}}
	\FOR{$n=N_1+1$ to $N$}
	\STATE Receive an initial state $x_1^n$.
	\FOR{$h = 1, \ldots, H$}
	\STATE Take greedy action $a_h^{n} = \argmax_{a}Q_{\hat{w}_{h}}(x_h^{n}, a)$ and transit to $x^{n}_{h+1}$. 
	\ENDFOR
	\ENDFOR
	\end{algorithmic}
\end{algorithm}
}

Next we derive the regret guarantee for the online Lasso-FQI under the sparse linear MDP model. The proof is deferred to Appendix \ref{sec:lasso_FQI}. We need a notion of restricted eigenvalue that is common in high-dimensional statistics \citep{bickel2009simultaneous, buhlmann2011statistics}. 
\begin{definition}[{\bf Restricted eigenvalue}]\label{def:RE}
Given a positive semi-definite matrix $Z\in\mathbb R^{d\times d}$ and integer $s\geq 1$, define the restricted minimum eigenvalue of $Z$ as $C_{\min}(Z, s):=$
\begin{equation*}
\min_{\cS\subset [d], |\cS|\leq s}\min_{\bbeta\in\mathbb R^d}\left\{\frac{\langle \bbeta, Z\bbeta \rangle}{\|\bbeta_{\cS}\|_2^2}:  \|\bbeta_{\cS^c}\|_{1}\leq 3\|\bbeta_{\cS}\|_{1}\right\} \,.
\end{equation*}
\end{definition}

\begin{theorem}[\bf Regret bound for online Lasso-FQI]\label{thm:regret_lasso_FQI}
Suppose the episodic MDP is $(s, \phi)$-sparse as defined in Definition \ref{def:sparse_MDP} and $\|\phi(x,a)\|_{\infty}\leq 1$ for any $(x,a)\in\cX\times \cA$. Assume the learner has oracle access to an exploratory policy $\pi_e$ defined in Definition \ref{def:good_exploratory} and $C_{\min}(\Sigma^{\pi_e}, s)$ is a strictly positive universal constant 
independent of $N$ and $d$.  Choose the regularization parameter $\lambda_1=H\sqrt{\log (2d)/N}$ and the number of episodes in the exploration phase $N_1$ as
\begin{equation*}
    N_1 = \left(\frac{2048s^2H^4N^2}{C_{\min}(\Sigma^{\pi_e}, s)^2}\log(2dH/\delta)\right)^{\tfrac{1}{3}}\,.
\end{equation*}
With probability $1-\delta$, the cumulative regret of online Lasso-FQI satisfies:
\begin{equation}\label{eqn:regret_FQI}
    R_N \leq 2\left(\frac{2048\log(2dH/\delta)}{C_{\min}(\Sigma^{\pi_e}, s)^2}\right)^{\tfrac{1}{3}}H^{\tfrac{4}{3}}s^{\tfrac{2}{3}}N^{\tfrac{2}{3}}\,.
\end{equation}
\end{theorem}

\begin{remark}
The condition that requires $C_{\min}(\Sigma^{\pi_e}, s)$ being dimension-free is weaker than  requiring $\sigma_{\min}(\Sigma^{\pi_e})$ being dimension-free since $C_{\min}(\Sigma^{\pi_e}, s)\geq\sigma_{\min}(\Sigma^{\pi_e})>0$. 
\end{remark}

With oracle access of an exploratory policy, we obtain a dimension-free sub-linear regret bound. Without  oracle access of such an exploratory policy, Theorem \ref{thm:lower_bound} implies a linear regret lower bound. On the other hand, without considering the sparsity, solving the MDP will suffer linear regret in the high-dimensional regime due to the well-known $\Omega( d\sqrt{N})$ lower bound. In summary, we emphasize that in high-dimensional regime, exploiting the sparsity to reduce the regret needs an exploratory policy but finding the exploratory policy is as hard as solving the MDP itself - an irresolvable ``chicken and egg'' problem.

\section{Comparsion with contextual bandits}
In this section we investigate the difference between online RL and linear contextual bandits. When the planning horizon $H=1$, the episodic MDP becomes to a contextual bandit. Specifically, consider a sparse linear contextual bandit. At $n$th episode, the environment generates a context $x_n$ i.i.d from a distribution $\xi_0$. The learner chooses an action $a_n\in\cA$ and receives a reward:
\begin{equation*}
    Y_n =\phi(x_n, a_n)^{\top} \theta + \eta_n\,,
\end{equation*}
where $(\eta_n)_{n=1}^N$ is a sequence of independent standard Gaussian random variables and $\theta\in\mathbb R^d$ is a $s$-sparse unknown parameter vector.

We define an analogous exploratory policy as in Definition \ref{def:good_exploratory}: for an exploratory policy $\pi_e$ in a linear contextual bandit, it will satisfy 
\begin{equation*}
    \sigma_{\min}(\Sigma^{\pi_e}) =\sigma_{\min}\Big(\mathbb E^{\pi_e}\Big[\phi(x_n, a_n)\phi(x_n, a_n)^{\top}\Big]\Big)>0\,,
\end{equation*}
where $x_n\sim \xi_0$ and $a_n\sim \pi_e(\cdot|x_n)$. In episodic MDPs, since the MDP transition kernel is unknown, we can not find the exploratory policy without solving the MDP. However, in linear contextual bandits, as long as there exists an exploratory policy and the context distribution is known, we can obtain the exploratory policy by solving the following optimization problem:
\begin{equation*}
   \max_{\pi}\sigma_{\min}\Big(\mathbb E_{x\sim\xi_0, a\sim \pi(\cdot|x)}\Big[\phi(x,a)\phi(x,a)^{\top}\Big]\Big)\,.
\end{equation*}
Thus, there is no additional cost of the regret to obtain the exploratory policy. Note that assuming known context distribution is much weaker than assuming known MDP transition kernel  since we can learn the context distribution very quickly online. Following the rest step of online Lasso-FQI in Algorithm \ref{alg:estc}, we can replicate the $\tilde{O}(s^{2/3}N^{2/3})$ regret upper bound without oracle access of the exploratory policy.

\section{Discussion}
In this paper, we provide the first investigation of online sparse RL in the high-dimensional regime. In general, exploiting the sparsity to minimize the regret is hard without further assumptions. This also highlights some fundamental differences of sparse learning between online RL and supervised learning or contextual bandits.

\paragraph{Acknowledgement}
We greatly thank Akshay Krishnamurthy for pointing out the FLAMBE \citep{AgKaKrSu20} work and the role of action space size, and Yasin Abbasi-Yadkori for proofreading.
\newpage
\appendix

\section{Proof of Theorem \ref{thm:regret_lasso_FQI}}\label{sec:lasso_FQI}
\begin{proof}
In this section, we prove the regret bound of online Lasso fitted-Q-iteration. Recall that $\pi_e$ is an exploratory policy that satisfies Definition \ref{def:good_exploratory}, e.g.,
\begin{equation*}
\sigma_{\min}\left(\mathbb E^{\pi_e}\left[\frac{1}{H}\sum_{h=1}^{H}\phi(x_h, a_h)\phi(x_h, a_h)^{\top}\right]\right)>0 \,,
\end{equation*}
where $x_1\sim \xi_0, a_h\sim \pi(\cdot|x_h), x_{h+1}\sim P(\cdot|x_h, a_h)$ and $\mathbb E^{\pi_e}$ denotes expectation over the sample path generated under policy $\pi_e$. Recall that $N_1$ is the number of episodes in exploration phase that will be specified later. Denote $\pi_{N_1}$ as the greedy policy with respect to the estimated Q-value calculated from the Lasso fitted-Q-iteration in Algorithm \ref{alg:estc}. According to the design of Algorithm \ref{alg:estc}, we keep using $\pi_{N_1}$ for the remaining $N-N_1$ episodes after exploration phase.
From the definition of the cumulative regret in Eq.~\eqref{def:regret}, we decompose $R_N$ according to the exploration phase and exploitation phase:
\begin{equation*}
    \begin{split}
        R_N = \sum_{n=1}^N\Big(V_1^*(x_1^n) - V_1^{\pi_n}(x_1^n)\Big) = \underbrace{\sum_{n=1}^{N_1}\Big(V_1^*(x_1^n) - V_1^{\pi_e}(x_1^n)\Big)}_{I_1:\text{ regret during exploring}} + \underbrace{\sum_{n=N_1+1}^N\Big(V_1^*(x_1^n) - V_1^{\pi_{N_1}}(x_1^n)\Big)}_{I_2:\text{regret during exploiting}}\,.
    \end{split}
\end{equation*}
 Since we assume $r\in[0, 1]$, from the definition of value functions, it is easy to see $0\leq V_1^*(x), V_1^{\pi_e}(x)\leq H$ for any $x\in\cX$. Thus, we can upper bound $I_1$ by 
\begin{equation}\label{eqn:regret_bound_I1}
    I_1\leq N_1H.
\end{equation}

To bound $I_2$, we will bound $\|V_1^* - V_1^{\pi_{N_1}}\|_{\infty}$ first using the following lemma. The detailed proof is deferred to Lemma \ref{sec:proof_estimation_error}. Recall that $C_{\min}(\Sigma^{\pi_e},s)$ is the restricted eigenvalue in Definition \ref{def:RE} and we split the exploratory dataset into $H$ folds with $R$ episodes per fold.
\begin{lemma}\label{lemma:estimation_error}
Suppose the number of episodes in the exploration phase satisfies
\begin{equation*}
    N_1\geq \frac{C_1s^2H\log(3d^2/\delta)}{C_{\min}(\Sigma^{\pi_e}, s)},
\end{equation*}
for some sufficiently large constant $C_1$ and   $\lambda_1=H\sqrt{\log (2d/\delta)/(RH)}$. Then we have with probability at least $1-\delta$,
\begin{eqnarray*}
    \big\|V_1^{\hat{\pi}_{N_1}} - V_1^*\big\|_{\infty}\leq  \frac{32\sqrt{2}sH^3}{C_{\min}(\Sigma^{\pi_e}, s)}\sqrt{\frac{\log(2dH/\delta)}{N_1}}\,.
\end{eqnarray*}
\end{lemma}

According to Lemma \ref{lemma:estimation_error}, we have 
\begin{equation}\label{eqn:regret_bound_I2}
\begin{split}
    I_2\leq N 
    \big\|V_1^{\hat{\pi}_{N_1}} - V_1^*\big\|_{\infty}\leq N \frac{32\sqrt{2}sH^3}{C_{\min}(\Sigma^{\pi_e}, s)}\sqrt{\frac{\log(2dH/\delta)}{N_1}}\,.
    \end{split}
\end{equation}
Putting the regret bound during exploring (Eq.~\eqref{eqn:regret_bound_I1}) and the regret bound during exploiting (Eq.~\eqref{eqn:regret_bound_I2}), we have 
\begin{equation*}
    R_N \leq N_1H+N \frac{32\sqrt{2}sH^3}{C_{\min}(\Sigma^{\pi_e}, s)}\sqrt{\frac{\log(2dH/\delta)}{N_1}}.
\end{equation*}
We optimize $N_1$ by letting 
\begin{equation}\label{eqn:optimal_N1}
    N_1H=N \frac{32\sqrt{2}sH^3}{C_{\min}(\Sigma^{\pi_e}, s)}\sqrt{\frac{\log(2dH/\delta)}{N_1}} 	\Rightarrow N_1 = \left(\frac{2048s^2H^4N^2}{C_{\min}(\Sigma^{\pi_e}, s)^2}\log(2dH/\delta)\right)^{1/3}\,.
\end{equation}
With this choice of $N_1$, we have with probability at least $1-\delta$
\begin{equation*}
    R_N \leq 2H\left(\frac{2048s^2H^4N^2}{C_{\min}(\Sigma^{\pi_e}, s)^2}\log(2dH/\delta)\right)^{1/3}\,.
\end{equation*}

\end{proof}
\begin{remark}
The optimal choice of $N_1$ in Eq.~\eqref{eqn:optimal_N1} requires the knowledge of $s$ and $C_{\min}(\Sigma, s)$ that is typically not available in practice. Thus, we can choose a relatively conservative $N_1$ as 
\begin{equation*}
    N_1 = \left(512H^4N^2\log(2dH/\delta)\right)^{1/3}\,,
\end{equation*}
such that 
\begin{equation*}
     R_N \leq  4\frac{s}{C_{\min}(\Sigma^{\pi_e}, s)}H\left(512s^2H^4N^2\log(2dH/\delta)\right)^{1/3}\,.
\end{equation*}
\end{remark}
\section{Additional proofs}
\subsection{Proof of Claim \ref{claim:regret_lower}}\label{sec:claim_regret_lower}
\begin{proof}
We prove the first part. Recall that 
   \begin{equation*}
        \cD_k = \left\{\sum_{n=1}^{\tau_k-1} \sum_{j=1}^{s-1}\varphi_j(x_{\text{u}}, A_2^n)\leq \frac{(\tau_k-1)(s-1)}{2}\right\} \,.
\end{equation*}
 To simplify the notation, we write $\varphi_{n}$ short for $\varphi(x_{\text{u}}, A_2^n)$ and $\varphi_{nj}$ as its $j$th coordinate.
From Eq.~\eqref{eqn:R}, we have 
\begin{equation*}
    \begin{split}
         R_N(\cM_k) &\geq  (H-1) \mathbb E_k\Big[\Big((\tau_k-1)(s-1)\varepsilon - \sum_{n=1}^{\tau_k}\sum_{j=1}^{s-1} \varphi_{nj}\varepsilon\Big)\mathbb I(\cD_k)\Big]\\
         &\geq \frac{Hs\varepsilon}{8}\mathbb E_k\Big[\frac{(\tau_k-1)(s-1)\varepsilon}{2}\mathbb I(\cD_k)\Big].
    \end{split}
\end{equation*}

Second, we derive a regret lower bound of alternative MDP $\tilde{\cM}_k$. Define $\tilde{a}^* = \argmax_{a_j^{\text{u}}\in\cA_2} \varphi(x_{\text{u}}, a_j^{\text{u}})^{\top}\tilde{\theta}^{(k)}$ as the optimal action when the learner is at state $x_{\text{u}}$ in MDP $\cM_k$. By a similar decomposition in Eq.~\eqref{eqn:R},
\begin{equation}\label{eqn:decom5}
    \begin{split}
        R_N(\tilde{\cM}_k) &\geq (H-1)\Big(\tilde{\mathbb E}_{k}\Big[\sum_{n=1}^{\tau_k-1}\langle \varphi(x_{\text{u}}, \tilde{a}^*), \tilde{\theta}^{(k)}\rangle\Big]-\tilde{\mathbb E}_{k}\Big[\sum_{n=1}^{\tau_k-1}\langle \varphi_{n}, \tilde{\theta}^{(k)}\rangle \Big]\Big)\\
        &= (H-1)\tilde{\mathbb E}_{k}\Big[2\tau_k(s-1)\varepsilon - \sum_{n=1}^{\tau_k-1}\langle\varphi_{n}, \tilde{\theta}^{(k)}\rangle \Big]\,.
    \end{split}
\end{equation}
 Next, we will find an upper bound for $\sum_{n=1}^{\tau_k-1}\langle \varphi_{n}, \tilde{\theta}^{(k)}\rangle$. From the definition of $\tilde{\theta}^{(k)}$ in Eq.~\eqref{def:x_tilde}, 
\begin{equation}\label{eqn:decom3}
\begin{split}
    \sum_{n=1}^{\tau_k-1}\langle \varphi_{n}, \tilde{\theta}^{(k)}\rangle &= \sum_{n=1}^{\tau_k-1}\langle \varphi_{n}, \theta + 2\varepsilon\tilde{z}^{(k)}\rangle\\
     &= \sum_{n=1}^{\tau_k-1}\langle\varphi_{n}, \theta\rangle + 2\varepsilon \sum_{n=1}^{\tau_k-1}\langle \varphi_{n}, \tilde{z}^{(k)}\rangle\\
     &\leq\sum_{n=1}^{\tau_k-1}\langle \varphi_{n}, \theta\rangle+2\varepsilon\sum_{n=1}^{\tau_k-1}\sum_{j\in\supp(\tilde{z}^{(k)})}|\varphi_{nj}|,
\end{split}
\end{equation}
where the last inequality is from the definition of $\tilde{z}^{(k)}$ in Eq.~\eqref{def:x_tilde}.
To bound the first term, we have
\begin{equation}\label{eqn:decom2}
    \begin{split}
        \sum_{n=1}^{\tau_k-1} \langle\varphi_{n}, \theta \rangle & = \sum_{n=1}^{\tau_k-1} \sum_{j=1}^{s-1}\varphi_{nj}\varepsilon \leq \varepsilon \sum_{n=1}^{\tau_k-1} \sum_{j=1}^{s-1}|\varphi_{nj}|.
    \end{split}
\end{equation}
Since all the $\varphi_{n}$ come from feature set $\cS$ which is a $(s-1)$-sparse set, we have 
    \begin{equation*}
    \sum_{n=1}^{\tau_k-1} \sum_{j=1}^d |\varphi_{nj}|= (s-1)(\tau_k-1),
\end{equation*}
Noting that the first $(s-1)$ coordinate of $\tilde{z}^{(k)}$ is 0, it implies 
\begin{equation}\label{eqn:decom4}
\begin{split}
    &\sum_{n=1}^{\tau_k-1} \Big(\sum_{j=1}^{s-1}|\varphi_{nj}|+\sum_{j\in\supp(\tilde{z}^{(k)})}|\varphi_{nj}|\Big)\leq \sum_{n=1}^{\tau_k-1} \sum_{j=1}^d |\varphi_{nj}| =(s-1)(\tau_k-1),\\
    &\sum_{n=1}^{\tau_k-1} \sum_{j=1}^{s-1}|\varphi_{nj}|\leq (s-1)(\tau_k-1)-\sum_{n=1}^{\tau_k-1} \sum_{j\in\supp(\tilde{z}^{(k)})}|\varphi_{nj}|.
\end{split}
\end{equation}
Combining with Eq.~\eqref{eqn:decom2},
\begin{equation*}
\begin{split}
    \sum_{n=1}^{\tau_k-1} \langle \varphi_{n}, \theta \rangle \leq \varepsilon \Big((s-1)(\tau_k-1) - \sum_{n=1}^{\tau_k-1}\sum_{j\in\supp(\tilde{z}^{(k)})}|\varphi_{nj}|\Big)
\end{split}
\end{equation*}
Plugging the above bound into Eq.~\eqref{eqn:decom3}, it holds that 
\begin{equation}\label{eqn:decom6}
     \sum_{n=1}^{\tau_k-1}\langle \varphi_{n}, \tilde{\theta}\rangle \leq \varepsilon (s-1)(\tau_k-1) + \varepsilon \sum_{n=1}^{\tau_k}\sum_{j\in\supp(\tilde{z}^{(k)})}|\varphi_{nj}|.
\end{equation}
When event $\cD_k^c$ (the complement event of $\cD_k$) happens, we have
\begin{equation*}
    \sum_{n=1}^{\tau_k-1} \sum_{j=1}^{s-1}|\varphi_{nj}|\geq \sum_{n=1}^{\tau_k-1} \sum_{j=1}^{s-1}\varphi_{nj}\geq \frac{(\tau_k-1)(s-1)}{2}.
\end{equation*}
Combining with Eq.~\eqref{eqn:decom4}, we have under event $\cD_k^c$, 
\begin{equation}\label{eqn:decom7}
    \sum_{n=1}^{\tau_k-1}\sum_{j\in\supp(\tilde{z}^{(k)})}|\varphi_{nj}| \leq \frac{(\tau_k-1)(s-1)}{2}.
\end{equation}
Putting Eqs.~\eqref{eqn:decom5}, \eqref{eqn:decom6}, \eqref{eqn:decom7} together, it holds that
\begin{equation}\label{eqn:bound_R_alt}
     R_{N}(\tilde{\cM}_k) \geq (H-1)\tilde{\mathbb E}_{k}\Big[ \frac{(\tau_k-1)(s-1)\varepsilon}{2}\mathbb I(\cD_k^c)\Big]\,. 
\end{equation}
Putting the lower bounds of $R_{N}(\cM_k)$ and $R_{N}(\tilde{\cM}_k)$ together, we have 
\begin{equation*}
\begin{split}
      R_{N}(\cM_k) + R_{N}(\tilde{\cM}_k) &\geq (H-1)\Big(\mathbb E_k\Big[\frac{(\tau_k-1)(s-1)\varepsilon}{2}\mathbb I(\cD_k)\Big] + \tilde{\mathbb E}_{k}\Big[ \frac{(\tau_k-1)(s-1)\varepsilon}{2}\mathbb I(\cD_k^c)\Big]\Big)\\
      &=\frac{Hs\varepsilon}{8}\Big(\mathbb E_k\Big[\tau_k\Big(\mathbb I(\cD_k)+\mathbb I(\cD_k^c)\Big)\Big]+\tilde{\mathbb E}_{k}[\tau_k\mathbb I(\cD_k^c)]-\mathbb E_{k}[\tau_k\mathbb I(\cD_k^c)]\Big)\\
      &=\frac{Hs\varepsilon}{8}\Big(\mathbb E_k[\tau_k]+\tilde{\mathbb E}_{k}[\tau_k\mathbb I(\cD_k^c)]-\mathbb E_{k}[\tau_k\mathbb I(\cD_k^c)]\Big).
\end{split}
\end{equation*}

This ends the proof.
\end{proof}
\subsection{Proof of Claim \ref{claim:KL_bound}}\label{sec:claim_KL}
\begin{proof}
The KL-calculation is inspired by \cite{jaksch2010near}, but with novel stopping time argument.
Denote the state-sequence up to $n$th episode, $h$th step  as $\mathbb S^n_h = \{S_1^1,\ldots,S_H^1, \ldots, S_1^n,\ldots,S_h^n\}$ and write $\cX^{n}_h = \{x_0, x_{\text{i}},x_{\text{u}}, x_{\text{g}}, x_{\text{b}}\}^{(n-1)H+h}$. 
For a fixed policy $\pi$ interacting with the environment for $n$ episodes, we denote $\mathbb P_{k}(\cdot)$ as the distribution over $\mathbb S^n$, where $S_1^n = x_0$, $A_h^n\sim \pi(\cdot|S_h^n)$, $S_{h+1}^n \sim \mathbb P_{k}(\cdot|S_h^n, A_h^n)$. Let $\EE_{k}$ denote the expectation w.r.t. distribution $\mathbb P_{k}$. By the chain rule, we can decompose the KL divergence as follows:
\begin{equation}\label{eqn:KL_decom}
    \text{KL}(\tilde{\mathbb P}_{k}\| \mathbb P_{k}) = \mathbb E\left[\sum_{n=1}^{\tau_k-1}\sum_{h=1}^{H}\text{KL}\Big[\tilde{\mathbb P}_{k}(S_{h+1}^n|\mathbb S_h^n)\Big\| \mathbb P_{k}(S_{h+1}^n|\mathbb S_h^n)\Big] \right]\,.
\end{equation}
Given a random variable $x$, the KL divergence over two conditional probability distributions is defined as 
\begin{equation*}
    \text{KL}\big(p(y|x), q(y|x)\big) = \sum_{x}\sum_{y} p(x,y) \log\left(\frac{p(y|x)}{q(y|x)}\right) \,.
\end{equation*}
Then the KL divergence between $\tilde{\mathbb P}_{k}(S_{h+1}^n|\mathbb S_h^n)$ and $\mathbb P_{k}(S_{h+1}^n|\mathbb S_h^n)$ can be calculated as follows:
\begin{equation}\label{eqn:KL_bound1}
\begin{split}
    &\text{KL}\Big[\tilde{\mathbb P}_{k}(S_{h+1}^n|\mathbb S_h^n)\Big\| \mathbb P_{k}(S_{h+1}^n|\mathbb S_h^n)\Big] \\
    &=\sum_{\mathbb S_h^n\in\cX_h^n}\sum_{x\in\cX}\tilde{\mathbb P}_{k}(S_{h+1}^n=x, \mathbb S_h^n)\log\left(\frac{\tilde{\mathbb P}_{k}(S_{h+1}^n=x|\mathbb S_h^n)}{\mathbb P_{k}(S_{h+1}^n=x|\mathbb S_h^n)}\right)\\
    &
    = \sum_{\mathbb S_h^n\in\cX_h^n}\sum_{x\in\cX}\tilde{\mathbb P}_{k}(S_{h+1}^n=x| \mathbb S_h^n)\tilde{\mathbb P}_{k}(  \mathbb S_h^n)\log\left(\frac{\tilde{\mathbb P}_{k}(S_{h+1}^n=x| \mathbb S_h^n)}{\mathbb P_{k}(S_{h+1}^n=x| \mathbb S_h^n)}\right)\\
     & = \sum_{\mathbb S_{h-1}^n\in\cX_{h-1}^n}\tilde{\mathbb P}_{k}(\mathbb S_{h-1}^n)\sum_{x' \in \cX, a \in \cA}\tilde{\mathbb P}_{k}(S_h^n = x',A_h^n = a|\mathbb S_{h-1}^n) \\
     &\qquad\cdot \sum_{x\in \cX} \tilde{\mathbb P}_{k}(S_{h+1}^n = x|\mathbb S_{h-1}^n, S_h^n = x', A_h^n = a)\log\left(\frac{\tilde{\mathbb P}_{k}(S_{h+1}^n =x|\mathbb S_{h-1}^n, S_h^n = x', A_h^n = a )}{\mathbb P_{k}(S_{h+1}^n =x|\mathbb S_{h-1}^n, S_h^n = x', A_h^n = a)}\right)\,.
    \end{split}
\end{equation}
According to the construction of  $\cM_k$ and $\tilde{\cM}_k$, the learner will remain staying at the current state when $x'=x_{\text{g}}$ or $x_{\text{b}}$, that implies
$$
\tilde{\mathbb P}_{k}(S_{h+1}^n = x|\mathbb S_{h-1}^n, S_h^n = x', A_h^n = a) = \mathbb P_{k}(S_{h+1}^n = x|\mathbb S_{h-1}^n, S_h^n = x', A_h^n = a)\,.
$$
In addition, from the definition of stopping time $\tau_k$, the learner will never transit to the informative state $x_{\text{i}}$. Therefore,
\begin{equation*}
    \begin{split}
        &\text{KL}\Big[\tilde{\mathbb P}_{k}(S_{h+1}^n|\mathbb S_{h}^n)\Big\| \mathbb P_{k}(S_{h+1}^n|\mathbb S_{h}^n)\Big]\\
     =& \sum_{\mathbb S_{h-1}^n \in \cX^{t-1}}\tilde{\mathbb P}_{k}(\mathbb S_{h-1}^n)\sum_{x'=x_0, x_{\text{i}}, x_{\text{u}}}\sum_{a\in\cA}\tilde{\mathbb P}_{k}(S_h^n = x', A_h^n = a|\mathbb S_{h-1}^n) \\
     &\qquad\cdot \sum_{x \in \cX} \tilde{\mathbb P}_{k}(S_{h+1}^n = x|\mathbb S_{h-1}^n, S_h^n = x', A_h^n = a)\log\left(\frac{\tilde{\mathbb P}_{k}(S_{h+1}^n = x|\mathbb S_{h-1}^n, S_h^n = x', A_h^n = a)}{\mathbb P_{k}(S_{h+1}^n = x|\mathbb S_{h-1}^n, S_h^n = x', A_h^n = a)} \right)\\
     =&\sum_{a\in \cA_2}\tilde{\mathbb P}_{k}(S_h^n = x_{\text{u}}, A_h^n = a) \sum_{x= x_{\text{g}},  x_{\text{b}}} \tilde{\mathbb P}_{k}(S_{h+1}^n = x| S_h^n = x_{\text{u}}, A_h^n = a)\log\left(\frac{\tilde{\mathbb P}_{k}(S_{h+1}^n =x| S_h^n = x_{\text{u}}, A_h^n = a )}{\mathbb P_{k}(S_{h+1}^n =x| S_h^n = x_{\text{u}}, A_h^n = a )}\right)\\
     =&\sum_{a\in\cA_2}\tilde{\mathbb P}_{k}(S_h^n = x_{\text{u}}, A_h^n = a)\Big(\langle \varphi(x_{\text{u}}, a), \tilde{\theta}^{(k)} \rangle \log\Big(\frac{\langle \varphi(x_{\text{u}}, a), \tilde{\theta}^{(k)} \rangle}{\langle \varphi(x_{\text{u}}, a), \theta \rangle}\Big) + (1-\langle \varphi(x_{\text{u}}, a), \tilde{\theta}^{(k)} \rangle) \log\Big(\frac{1-\langle \varphi(x_{\text{u}}, a), \tilde{\theta}^{(k)} \rangle}{1-\langle \varphi(x_{\text{u}}, a), \theta \rangle}\Big)\Big)\,,
     \end{split}
     \end{equation*}
    where $\cA_2$ is the action set associated to state $x_{\text{u}}$.
Moreover, we will use Lemma \ref{lemma:supp1} to bound the above last term. Letting $q = \langle \varphi(x_{\text{u}}, a), \tilde{\theta}^{(k)} \rangle$ and $\epsilon = \langle \varphi(x_{\text{u}}, a),\theta- \tilde{\theta}^{(k)} \rangle$, it is easy to verify the conditions in Lemma \ref{lemma:supp1} as long as $\varepsilon \leq (10(s-1))^{-1}$. Then we have
\begin{equation*}
    \begin{split}     
    \text{KL}\Big[\tilde{\mathbb P}_{k}(S_{h+1}^n|\mathbb S_{h}^n)\Big\| \mathbb P_{k}(S_{h+1}^n|\mathbb S_{h}^n)\Big] &\leq \sum_{a\in\cA_2}\tilde{\mathbb P}_{k}(S_h^n = x_{\text{u}},A_h^n = a) \frac{2\langle \tilde{\theta}^{(k)}-\theta, \varphi(x_{\text{u}}, a)\rangle^2}{\langle \tilde{\theta}^{(k)}, \varphi(x_{\text{u}}, a)\rangle} \\
&= \sum_{a\in\cA_2}\tilde{\mathbb P}_{k}(S_h^n = x_{\text{u}}, A_h^n = a) \frac{8\varepsilon^2\langle \tilde{z}^{(k)}, \varphi(x_{\text{u}}, a)\rangle^2}{\langle \tilde{\theta}, \varphi(x_{\text{u}}, a)\rangle} \,.
    \end{split}
\end{equation*}
Back to the KL-decomposition in Eq.~\eqref{eqn:KL_decom}, we have 
\begin{equation*}
  \text{KL}(\tilde{\mathbb P}_{k}\| \mathbb P_{k})\leq 8\varepsilon^2\tilde{\mathbb E}_k\Big[\sum_{n=1}^{\tau_k-1}\langle \varphi(x_{\text{u}}, A_2^n), \tilde{z}\rangle^2\Big]\,.
\end{equation*}
To simplify the notations, we let $\varphi_n = \varphi(x_{\text{u}}, A_2^n)$.

Next, we use a simple argument ``minimum is always smaller than the average". We decompose the following summation over action set $\cS'$ defined in Eq.~\eqref{def:S_prime},
\begin{equation*}
    \begin{split}
        \sum_{z\in\cS'} \sum_{n=1}^{\tau_k-1} \langle \varphi_n, z \rangle^2&= \sum_{z\in\cS'} \sum_{n=1}^{\tau_k-1} \Big(\sum_{j=1}^dz_j \varphi_{nj}\Big)^2\\
      &=  \sum_{z\in\cS'} \sum_{n=1}^{\tau_k-1} \Big(\sum_{j=1}^d\big(z_j \varphi_{nj}\big)^2 + 2\sum_{i<j}z_iz_j\varphi_{ni}\varphi_{nj}\Big).
    \end{split}
\end{equation*}
We bound the above two terms separately. 
To bound the first term, we observe that 
    \begin{equation}\label{eqn:decom_1}
    \begin{split}
         &\sum_{z\in\cS'} \sum_{n=1}^{\tau_k-1} \sum_{j=1}^d\big(z_j \varphi_{nj}\big)^2 = \sum_{z\in\cS'} \sum_{n=1}^{\tau_k-1}\sum_{j=1}^d|z_j \varphi_{nj}|,
    \end{split}
    \end{equation}
    since both $z_j, \varphi_{nj}$ can only take $-1, 0, +1$. In addition, $\sum_{t=1}^{\tau_k-1} \sum_{j=1}^d |\varphi_{nj}|= (s-1)(\tau_k-1).$
Since $z\in\cS'$ that is $(s-1)$-sparse, we have $\sum_{j=1}^d|z_j \varphi_{nj}| \leq s-1$. Therefore, we have
\begin{equation}\label{eqn:bound1}
\begin{split}
    \sum_{z\in\cS'}\sum_{n=1}^{\tau_k-1} \sum_{j=1}^d |z_j\varphi_{nj}|\leq (s-1)(\tau_k-1) \binom{d-s-1}{s-2}.
\end{split}
\end{equation}
Putting Eqs.~\eqref{eqn:decom_1} and \eqref{eqn:bound1}  together,
\begin{equation}\label{eqn:sum1}
     \sum_{z\in\cS'} \sum_{n=1}^{\tau_k-1} \sum_{j=1}^d\big(z_j \varphi_{nj}\big)^2 \leq (s-1)(\tau_k-1) \binom{d-s-1}{s-2}.
\end{equation}

To bound the second term, we observe 
\begin{equation*}
    \sum_{z\in\cS'} \sum_{n=1}^{\tau_k-1}  2\sum_{i<j}z_iz_j\varphi_{ni}\varphi_{nj}=2\sum_{n=1}^{\tau_k-1} \sum_{i<j} \sum_{z\in\cS'} z_iz_j\varphi_{ni}\varphi_{nj}.
\end{equation*}
From the definition of $\cS'$, $z_iz_j$ can only take values of $\{1*1, 1*-1, -1*1, -1*-1, 0\}$. This symmetry implies 
\begin{equation*}
    \sum_{z\in\cS'}z_iz_j\varphi_{ni}\varphi_{nj} = 0,
\end{equation*}
which implies 
\begin{equation}\label{eqn:sum2}
     \sum_{z\in\cS'} \sum_{n=1}^{\tau_k-1}  2\sum_{i<j}z_iz_j\varphi_{ni}\varphi_{nj} = 0.
\end{equation}

Combining Eqs.~\eqref{eqn:sum1} and \eqref{eqn:sum2} together, we have
\begin{equation*}
\begin{split}
\sum_{z\in\cS'} \sum_{n=1}^{\tau_k-1} \langle \varphi_{n}, z \rangle^2&= \sum_{z\in\cS'} \sum_{n=1}^{\tau_k-1} \sum_{j=1}^d|z_j \varphi_{nj}|\leq (s-1)(\tau_k-1) \binom{d-s-1}{s-2}.
\end{split}
\end{equation*}
In the end, we use the fact that the minimum of $\tau_k-1$ points is always smaller than its average,
\begin{equation*}
\begin{split}
    \tilde{\mathbb E}_k\Big[\sum_{n=1}^{\tau_k-1} \langle \varphi_{n}, \tilde{z} \rangle^2\Big] &=  \min_{z\in\cS'}\tilde{\mathbb E}_k\Big[\sum_{n=1}^{\tau_k-1} \langle \varphi_{n}, z\rangle^2\Big]\\
     &\leq \frac{1}{|\cS'|}\sum_{z\in\cS'}\tilde{\mathbb E}_k\Big[\sum_{n=1}^{\tau_k-1} \langle \varphi_{n}, z\rangle^2\Big]\\
     &=\tilde{\mathbb E}_k\Big[\frac{1}{|\cS'|}\sum_{z\in\cS'} \sum_{n=1}^{\tau_k-1} \langle \varphi_{n}, z \rangle^2\Big]\\
    &\leq \frac{(s-1) \tilde{\mathbb E}_k[\tau_k-1]\binom{d-s-1}{s-2}}{\binom{d-s}{s-1}}\\
    &\leq \frac{(s-1)^2 \tilde{\mathbb E}_k[\tau_k-1]}{d}\,.
\end{split}
\end{equation*}
Therefore, we reach 
\begin{equation*}
  \text{KL}(\tilde{\mathbb P}_{k}\| \mathbb P_{k})\leq \frac{8\varepsilon^2(s-1)^2 \tilde{\mathbb E}_k[\tau_k-1]}{d}\leq \frac{8\varepsilon^2(s-1)^2 N}{d}\leq 8\varepsilon^2(s-1)^2 \,,
\end{equation*}
since we consider the data-poor regime that $N\leq d$. It is obvious to see $  \text{KL}(\mathbb P_{0}\| \mathbb P_{k})=0$ from Eq.~\eqref{eqn:KL_bound1}.
This ends the proof.
\end{proof}
\subsection{Proof of Lemma \ref{lemma:estimation_error}}\label{sec:proof_estimation_error}
\begin{proof}
Recall that in the learning phase, we split the data collected in the exploration phase into $H$ folds and each fold consists of $R$ episodes or $RH$ sample transitions. For the update of each step $h$, we use a fresh fold of samples.

\textbf{Step 1.}  We verify that the execution of Lasso fitted-Q-iteration is equivalent to the approximate value iteration.  Recall that a generic Lasso estimator with respect to a function $V$ at step $h$ is defined in Eq.~\eqref{eqn:lasso} as 
\begin{equation*}
\hat{w}_{h}(V) = \argmin_{w\in\mathbb R^d}\Big(\frac{1}{RH}\sum_{i=1}^{RH}\Big(\Pi_{[0,H]}V(x_i^{(h)'}) -\phi(x_i^{(h)},a_i^{(h)})^{\top}w\Big)^2 + \lambda_1 \|w\|_1\Big).
\end{equation*}
Denote $V_w(x)=\max_{a\in\cA}(r(x,a)+ \phi(x,a)^{\top}w)$. For simplicity, we write $\hat{w}_h := \hat{w}_{h}(V_{\hat{w}_{h+1}})$ for short. Define an approximate Bellman optimality operator $\hat{\cT}^{(h)}:\cX\to\cX$ as:
\begin{equation}\label{eqn:T1}
    [\hat{\cT}^{(h)}V](x) := \max_a\Big[r(x,a) + \phi(x,a)^{\top}\hat{w}_h(V)\Big].
\end{equation}
Note this $\hat{\cT}^{(h)}$ is a randomized operator that only depends data from $h$th fold.
The Lasso fitted-Q-iteration in learning phase of Algorithm \ref{alg:estc} is equivalent to the following approximate value iteration:
\begin{equation}\label{eqn:api}
    [\hat{\cT}^{(h)}\Pi_{[0,H]}V_{\hat{w}_{h+1}}](x) = \max_a\Big[r(x,a) + \phi(x,a)^{\top}\hat{w}_h\Big] = \max_a Q_{\hat{w}_h}(x,a) = V_{\hat{w}_h}(x).
\end{equation}
Recall that the true Bellman optimality operator in state space $\cT:\cX\to\cX$ is defined as 
\begin{equation}\label{eqn:true_bell}
    [\cT V](x) := \max_a\Big[r(x,a)+ \sum_{x'}P(x'|x,a)V(x')\Big].
\end{equation}

\paragraph{Step 2.} We verify that the true Bellman operator on $\Pi_{[0,H]}V_{\hat{w}_{h+1}}$ can also be written as a linear form. From Definition  \ref{def:sparse_MDP}, there exists some functions $\psi(\cdot) = (\psi_k(\cdot))_{k\in\cK}$ such that  for every $x,a,x'$, the transition function can be represented as 
\begin{equation}\label{eqn:linear_MDP}
    P(x' |x,a) = \sum_{k\in \cK} \phi_k(x,a) \psi_k(x'),
\end{equation}
where $\cK\subseteq [d]$ and $|\cK|\leq s$. For a vector $\bar{w}_h\in\mathbb R^d$, we define its $k$th coordinate as
\begin{equation}\label{eqn:def_w_bar}
    \bar{w}_{h,k} = \sum_{x'} \Pi_{[0,H]} V_{\hat{w}_{h+1}}(x')\psi_k(x'), \ \text{if} \ k\in\cK,
\end{equation}
and $\bar{w}_{h,k}=0$ if $k\notin \cK$. By the definition of true Bellman optimality operator in Eq.~\eqref{eqn:true_bell} and Eq.~\eqref{eqn:linear_MDP},
\begin{eqnarray}\label{eqn:true_bell_linear}
     [\cT \Pi_{[0,H]}V_{\hat{w}_{h+1}}](x) &=& \max_a\Big[r(x,a)+ \sum_{x'}P(x'|x,a)\Pi_{[0,H]}V_{\hat{w}_{h+1}}(x')'\Big]\nonumber\\
     &=& \max_a\Big[r(x,a)+ \sum_{x'}\phi(x,a)^{\top}\psi(x')\Pi_{[0,H]}V_{\hat{w}_{h+1}}(x')'\Big]\nonumber\\
     &=& \max_a\Big[r(x,a)+ \sum_{x'}\sum_{k\in\cK}\phi_k(x,a)\psi_k(x')\Pi_{[0,H]}V_{\hat{w}_{h+1}}(x')'\Big]
    \nonumber\\
    &=& \max_a\Big[r(x,a)+ \sum_{k\in\cK}\phi_k(x,a)\sum_{x'}\psi_k(x')\Pi_{[0,H]}V_{\hat{w}_{h+1}}(x')'\Big]
    \nonumber\\
      &=& \max_a\Big[r(x,a)+\phi(x,a)^{\top}\bar{w}_h\Big]\,.
\end{eqnarray}
We interpret $\bar{w}_h$ as the ground truth of the Lasso estimator in Eq.~\eqref{eqn:lasso} at step $h$ in terms of the following sparse linear regression:
\begin{eqnarray}\label{eqn:regression_eva}
    \Pi_{[0,H]}V_{\hat{w}_{h+1}}(x_i') = \phi(x_i, a_i)^{\top}\bar{w}_h+\varepsilon_i, i=1\ldots, RH,
\end{eqnarray}
where $\varepsilon_i =  \Pi_{[0,H]}V_{\hat{w}_{h+1}}(x_i') -\phi(x_i, a_i)^{\top}\bar{w}_h$. Define the filtration $\cF_i$ generated by $\{(x_1, a_1),\ldots, (x_{i}, a_{i})\}$ and also the data in folds $h+1$ to $H$. By the definition of $V_{\hat{w}_{h+1}}$ and $\bar{w}_h$, we have 
\begin{equation*}
    \begin{split}
        \mathbb E[\varepsilon_{i}|\cF_i] &=   \mathbb E\big[\Pi_{[0,H]}V_{\hat{w}_{h+1}}(x_i')|\cF_i\big] -\phi(x_i, a_i)^{\top}\bar{w}_h\\
        &= \sum_{x'}[\Pi_{[0,H]}  V_{\hat{w}_{h+1}}](x')P(x'|x_i, a_i) -\phi(x_i, a_i)^{\top}\bar{w}_h\\
        & = \sum_{k\in\cK}\phi_k(x_i, a_i)\sum_{x'}[\Pi_{[0,H]}V_{\hat{w}_{h+1}}](x')\psi_k(x')-\phi(x_i, a_i)^{\top}\bar{w}_h = 0.
    \end{split}
\end{equation*}
Therefore, $\{\varepsilon_i\}_{i=1}^{RH}$ is a sequence of martingale difference noises and $|\varepsilon_i|\leq H$ due to the truncation operator $\Pi_{[0,H]}$. 
The next lemma bounds the difference between $\hat{w}_h$ and $\bar{w}_h$ within $\ell_1$-norm.  
The proof is deferred to Appendix \ref{sec:proof_lasso_l1}.

\begin{lemma}\label{lemma:lasso_l1_bound}
Consider the sparse linear regression described in Eq.~\eqref{eqn:regression_eva}. Suppose the number of episodes used in step $h$ satisfies
\begin{equation*}
    R\geq \frac{C_1\log(3d^2/\delta)s^2}{C_{\min}(\Sigma^{\pi_e}, s)},
\end{equation*}
for some absolute constant $C_1>0.$ With the choice of $\lambda_1=H\sqrt{\log (2d/\delta)/(RH)}$, the following holds with probability at least $1-\delta$, 
\begin{equation}\label{eqn:lasso_error_bound}
  \big\|\hat{w}_h-\bar{w}_h\big\|_1\leq \frac{16\sqrt{2}s}{C_{\min}(\Sigma^{\pi_e}, s)}H\sqrt{\frac{\log(2d/\delta)}{RH}}.
\end{equation}
\end{lemma}

\paragraph{Step 3.} We start to bound $\|V_{\hat{w}_{h}}-V_h^*\|_{\infty}$ for each step $h$. By the approximate value iteration form Eq.~\eqref{eqn:api} and the definition of optimal value function,
\begin{equation}\label{eqn:V_decomp}
    \begin{split}
         \big\|V_{\hat{w}_{h}}-V_h^*\big\|_{\infty}&= \big\|\hat{\cT}^{(h)}\Pi_{[0,H]}V_{\hat{w}_{h+1}} - \cT V_{h+1}^*\big\|_{\infty} \\
   &=\big\|\hat{\cT}^{(h)}\Pi_{[0,H]}V_{\hat{w}_{h+1}} - \cT\Pi_{[0,H]}V_{\hat{w}_{h+1}} \big\|_{\infty} + \big\|\cT\Pi_{[0,H]}V_{\hat{w}_{h+1}} - \cT V_{h+1}^* \big\|_{\infty}.
    \end{split}
\end{equation}
The first term mainly captures the error between approximate Bellman optimality operator and true Bellman optimality operator. From linear forms Eqs.~\eqref{eqn:api} and \eqref{eqn:true_bell_linear}, it holds for any $x\in\cX$,
\begin{eqnarray}\label{eqn:approx_Bellman}
   &&[\hat{\cT}^{(h)}\Pi_{[0,H]}V_{\hat{w}_{h+1}}](x) - [\cT\Pi_{[0,H]}V_{\hat{w}_{h+1}}](x)\nonumber \\
   &=& \max_a\Big[r(x,a) + \phi(x,a)^{\top}\hat{w}_h\Big] - \max_a\Big[r(x,a)+ \phi(x,a)^{\top}\bar{w}_h\Big]\nonumber\\
   &\leq&  \max_a\big|\phi(x,a)^{\top}(\hat{w}_h-\bar{w}_h)\big|\nonumber\\
   &\leq&  \max_{a,x}\|\phi(x,a)\|_{\infty}\|\hat{w}_h-\bar{w}_h\|_1.
\end{eqnarray}
Applying Lemma \ref{lemma:lasso_l1_bound}, the following error bound holds with probability at least $1-\delta$, 
\begin{equation}\label{eqn:lasso_error_bound2}
  \big\|\hat{w}_h-\bar{w}_h\big\|_1\leq \frac{16\sqrt{2}s}{C_{\min}(\Sigma^{\pi_e}, s)}H\sqrt{\frac{\log(2d/\delta)}{RH}},
\end{equation}
where $R$ satisfies $ R\geq C_1\log(3d^2/\delta)s^2/C_{\min}(\Sigma^{\pi_e}, s).$

Note that the samples we use between phases are mutually independent. Thus Eq.~\eqref{eqn:lasso_error_bound2} uniformly holds for all $h\in[H]$ with probability at least $1-H\delta$. Plugging it into Eq.~\eqref{eqn:approx_Bellman}, we have for any stage $h\in[H]$,
\begin{eqnarray}\label{eqn:bound_1}
    \big\|\hat{\cT}^{(h)}\Pi_{[0,H]}V_{\hat{w}_{h+1}} - \cT\Pi_{[0,H]}V_{\hat{w}_{h+1}} \big\|_{\infty}\leq \frac{16\sqrt{2}s}{C_{\min}(\Sigma^{\pi_e}, s)}H\sqrt{\frac{\log(2dH/\delta)}{RH}},
\end{eqnarray}
holds with probability at least $1-\delta$. 

To bound the second term in Eq.~\eqref{eqn:V_decomp}, we observe that
\begin{equation}\label{eqn:bound_2}
\begin{split}
     \big\|\cT\Pi_{[0,H]}V_{\hat{w}_{h+1}} - \cT V_{h+1}^* \big\|_{\infty} &=\max_x\big|\cT\Pi_{[0, H]}V_{\hat{w}_{h+1}}(x)-\cT V_{h+1}^*(x)\big|\\
     &\leq \max_x\max_a\big|\sum_{x'} P(x'|x,a)\Pi_{[0, H]} V_{\hat{w}_{h+1}}(x')-\sum_{x'} P(x'|x,a)\Pi_{[0, H]} V_{h+1}^*(x')\big|\\
     &\leq \big\|\Pi_{[0,H]}V_{\hat{w}_{h+1}} - V_{h+1}^*\big\|_{\infty}\,.
\end{split}
\end{equation}
Plugging Eqs.~\eqref{eqn:bound_1} and \eqref{eqn:bound_2} into Eq.~\eqref{eqn:V_decomp}, it holds that
\begin{equation}\label{eqn:V_bound}
    \big\|V_{\hat{w}_{h}}-V_h^*\big\|_{\infty}\leq  \frac{16\sqrt{2}s}{C_{\min}(\Sigma^{\pi_e}, s)}H\sqrt{\frac{\log(2dH/\delta)}{RH}}+\big\|\Pi_{[0,H]}V_{\hat{w}_{h+1}} - V_{h+1}^*\big\|_{\infty}\,,
\end{equation}
with probability at least $1-\delta$. Recursively using Eq.~\eqref{eqn:V_bound}, the following holds with probability $1-\delta$,
\begin{align*}
         \big\|\Pi_{[0,H]}V_{\hat{w}_{1}}- V_1^*\big\|_{\infty}
         &\leq \big\|V_{\hat{w}_{1}}- V_1^*\big\|_{\infty}\\
    &=  \frac{16\sqrt{2}s}{C_{\min}(\Sigma^{\pi_e}, s)}H\sqrt{\frac{\log(2dH/\delta)}{RH}} +  \big\|\Pi_{[0,H]}V_{\hat{w}_{2}} -  V_2^* \big\|_{\infty}\\
    &\leq  \big\|\Pi_{[0,H]}V_{\hat{w}_{H+1}} -  V_{H+1}^* \big\|_{\infty} + H^2 \frac{16\sqrt{2}s}{C_{\min}(\Sigma^{\pi_e}, s)}\sqrt{\frac{\log(2dH/\delta)}{RH}}\\
    &=H^2 \frac{16\sqrt{2}s}{C_{\min}(\Sigma^{\pi_e}, s)}\sqrt{\frac{\log(2dH/\delta)}{RH}}\,,
\end{align*}
where the first inequality is due to that $\Pi_{[0,H]}$ can only make error smaller and the last inequality is due to $V_{\hat{w}_{H+1}}=V_{H+1}^*=0$. From Proposition 2.14 in \cite{bertsekas1995dynamic}, \begin{equation}\label{eqn:upper_bound}
     \big\|V_1^{\hat{\pi}_{N_1}} - V_1^*\big\|_{\infty} \leq H\big\|Q_{\hat{w}_1}-Q_1^*\big\|_{\infty} \leq 2H\big\|\Pi_{[0,H]}V_{\hat{w}_{1}}-V_1^*\big\|_{\infty}\,. 
\end{equation}
Putting the above together, we have with probability at least $1-\delta$,
\begin{eqnarray*}
    \big\|V_1^{\hat{\pi}_{N_1}} - V_1^*\big\|_{\infty}\leq  \frac{32\sqrt{2}sH^3}{C_{\min}(\Sigma^{\pi_e}, s)}\sqrt{\frac{\log(2dH/\delta)}{N_1}}\,,
\end{eqnarray*}
when the number of episodes in the exploration phase has to satisfy
\begin{equation*}
    N_1\geq \frac{C_1s^2H\log(3d^2/\delta)}{C_{\min}(\Sigma^{\pi_e}, s)},
\end{equation*}
for some sufficiently large constant $C_1$. This ends the proof.  
\end{proof}

\subsection{Proof of Lemma \ref{lemma:lasso_l1_bound}}\label{sec:proof_lasso_l1}
\begin{proof}
Denote the empirical covariance matrix induced by the exploratory policy $\pi_e$ and feature map $\phi$ as
\begin{equation*}
    \hat{\Sigma}^{\pi_e}:= \frac{1}{R}\sum_{r=1}^R \frac{1}{H}\sum_{h=1}^H \phi(x_h^{r}, a_h^{r})\phi(x_h^{r}, a_h^{r})^{\top} .
\end{equation*}
Recall that $\Sigma^{\pi_e}$ is the population covariance matrix induced by the exploratory policy $\pi_e$ defined in Eq.~\eqref{eqn:expected_cov} and feature map $\phi$ with $\sigma_{\min}(\Sigma^{\pi_e})>0$. From the definition of restricted eigenvalue in \eqref{def:RE} it is easy to verify $C_{\min}(\Sigma^{\pi_e}, s)\geq \sigma_{\min}(\Sigma^{\pi_e})>0$. For any $i,j\in[d]$, denote
\begin{equation*}
    v_{ij}^{r} = \frac{1}{H}\sum_{h=1}^{H}\phi_i(x_h^{r}, a_h^{r})\phi_j(x_h^{r}, a_h^{r})-\Sigma^{\pi_e}_{ij}.
\end{equation*}
It is easy to verify $\mathbb E[v_{ij}^{r}] = 0$ and $|v_{ij}^{r}|\leq 1$ since  we assume $\|\phi(x,a)\|_{\infty}\leq 1$. Note that samples between different episodes are independent. This implies $v_{ij}^{1}, \ldots,v_{ij}^{R}$ are independent. By standard Hoeffding's inequality (Proposition 5.10 in \cite{vershynin2010introduction}), we have 
\begin{equation*}
    \mathbb P\Big(\Big|\sum_{r=1}^Rv_{ij}^{r}\Big|\geq \delta\Big)\leq 3\exp\Big(-\frac{C_0\delta^2}{R}\Big),
\end{equation*}
for some absolute constant $C_0>0$. Applying an union bound over $i,j\in[d]$, we have
\begin{equation*}
\begin{split}
    &\mathbb P\Big(\max_{i,j}\Big|\sum_{r=1}^Rv_{ij}^{r}\Big|\geq \delta\Big)\leq 3d^2\exp\Big(-\frac{C_0\delta^2}{R}\Big)\\
    &\Rightarrow \mathbb P\Big(\big\|\hat{\Sigma}^{\pi_e} -\Sigma^{\pi_e}\big\|_{\infty}\geq \delta\Big)\leq 3d^2\exp\Big(-\frac{C_0\delta^2}{R}\Big).
\end{split}
\end{equation*}
It implies the following holds with probability $1-\delta$,
\begin{equation*}
    \big\|\hat{\Sigma}^{\pi_e} -\Sigma^{\pi_e}\big\|_{\infty}\leq \sqrt{\frac{\log (3d^2/\delta)}{R}}.
\end{equation*}
When the number of episodes $R\geq 32^2\log (3d^2/\delta)s^2/C_{\min}(\Sigma^{\pi_e}, s)^2$, the following holds with probability at least $1-\delta$,
\begin{equation*}
     \big\|\hat{\Sigma}^{\pi_e} -\Sigma^{\pi_e}\big\|_{\infty}\leq \frac{C_{\min}(\Sigma^{\pi_e}, s)}{32s}.
\end{equation*}
Next lemma shows that if the restricted eigenvalue condition holds for one positive semi-definite matrix $\Sigma_0$, then it holds with high probability for another positive semi-definite matrix $\Sigma_1$ as long as $\Sigma_0$ and $\Sigma_1$ are close enough in terms of entry-wise max norm.
\begin{lemma}[Corollary 6.8 in \citep{buhlmann2011statistics}]\label{lemma:eigen_concentration}
Let $\Sigma_0$ and $\Sigma_1$ be 
 two positive semi-definite block diagonal matrices. 
 Suppose that the restricted eigenvalue of $\Sigma_0$ satisfies $C_{\min}(\Sigma_0, s)>0$ and $\|\Sigma_1-\Sigma_0\|_{\infty}\leq C_{\min}(\Sigma_0, s)/(32s)$. 
 Then the restricted eigenvalue of $\Sigma_1$  satisfies $C_{\min}(\Sigma_1, s)>C_{\min}(\Sigma_0, s)/2$.
\end{lemma}
Applying Lemma \ref{lemma:eigen_concentration} with $\hat{\Sigma}^{\pi_e}$ and $\Sigma^{\pi_e}$, we have the restricted eigenvalue of $\hat{\Sigma}^{\pi_e}$ satisfies $C_{\min}(\hat{\Sigma}^{\pi_e}, s)>C_{\min}(\Sigma^{\pi_e}, s)/2$ with high probability.

Note that $\{\varepsilon_i\phi_j(x_i, a_i)\}_{i=1}^{RH}$ is also a martingale difference sequence and $|\varepsilon_i\phi_j(x_i, a_i)|\leq H$.  By Azuma-Hoeffding inequality,
\begin{equation*}
    \mathbb P\Big(\max_{j\in[d]}\Big|\frac{1}{RH}\sum_{i=1}^{RH}\varepsilon_i\phi_{j}(x_i, a_i)\Big|\leq H\sqrt{\frac{\log (2d/\delta)}{RH}}\Big)\geq 1-\delta.
\end{equation*}
Denote event $\cE$ as 
\begin{equation*}
    \cE = \Big\{\max_{j\in[d]}\Big|\frac{1}{RH}\sum_{i=1}^{RH}\varepsilon_i\phi_{j}(x_i, a_i)\Big|\leq \lambda_1\Big\}.
\end{equation*}
Then $\mathbb P(\cE)\geq 1-\delta$. Under event $\cE$, applying (B.31) in \cite{bickel2009simultaneous}, we have
\begin{equation*}
    \big\|\hat{w}_h-\bar{w}_h\big\|_1\leq \frac{16\sqrt{2}s\lambda_1}{C_{\min}(\Sigma^{\pi_e}, s)},
\end{equation*}
holds with probability at least $1-2\delta$. This ends the proof. 

\end{proof}

\section{Supporting lemmas}\label{sec:supporting_lemmas}
\begin{lemma}[Pinsker's inequality]
Denote $\xb = \{x_1, \ldots, x_T\}\in\cX^{T}$ as the observed states from step 1 to $T$. Then for any two distributions $P_1$ and $P_2$ over $\cX^{\top}$ and any bounded function $f:\cX^{\top}\to[0, B]$, we have 
\begin{equation*}
    \mathbb E_1f(\xb) - \mathbb E_2f(\xb)\leq \sqrt{\log 2/2}B\sqrt{\KL(P_2\|P_1)},
\end{equation*}
where $\mathbb E_1$ and $\mathbb E_2$ are expectations with respect to $P_1$ and $P_2$.
\end{lemma}
\begin{lemma}[Bretagnolle-Huber inequality]\label{lem:kl}
Let $\mathbb P$ and $\tilde{\mathbb P}$ be two probability measures on the same measurable space $(\Omega,\cF)$. Then for any event $\cD\in \cF$,
\begin{equation}\label{eqn:kl}
\mathbb P(\cD) + \tilde{\mathbb P}(\cD^c) \geq \frac{1}{2} \exp\left(-\text{KL}(\mathbb P, \tilde{\mathbb P})\right)\,,
\end{equation}
where $\cD^c$ is the complement event of $\cD$ ($\cD^c = \Omega\setminus \cD$) and $\text{KL}(\mathbb P, \tilde{\mathbb P})$ is the KL divergence between $\PP$ and $\tilde{\mathbb P}$, which is defined as $+\infty$, if $\PP$ is not absolutely continuous with respect to $\tilde{\mathbb P}$, and is $\int_\Omega d\PP(\omega) \log \frac{d\PP}{d\tilde{\mathbb P}}(\omega)$ otherwise.
\end{lemma}
The proof can be found in the book of \cite{Tsybakov:2008:INE:1522486}. When $\text{KL}(\mathbb P, \tilde{\mathbb P})$ is small, we may expect the probability measure $\mathbb P$ is close to the probability measure $\tilde{\mathbb P}$. Note that $\mathbb P(\cD) + \mathbb P(\cD^c)=1$. If $\tilde{\mathbb P}$ is close to $\mathbb P$, we may expect $\mathbb P(\cD)+\tilde{\mathbb P}(\cD^c)$ to be large.

\begin{lemma}[Divergence decomposition]\label{lem:inf-processing}
Let $\mathbb P$ and $\tilde{\mathbb P}$ 
be two probability measures on the sequence  $(A_1, Y_1,\ldots,A_n,Y_n)$ for a fixed
bandit policy $\pi$ interacting with a linear contextual bandit with standard Gaussian noise and parameters $\theta$ and $\tilde{\theta}$ respectively. Then the KL divergence of $\mathbb P$ and $\tilde{\mathbb P}$ can be computed  exactly and is given by
\begin{equation}\label{eqn:inf-processing}
\text{KL}(\mathbb P, \tilde{\mathbb P}) = \frac12 \sum_{x \in \mathcal A} \mathbb E[T_x(n)]\, \langle x, \theta - \tilde{\theta}\rangle^2\,,
\end{equation}
where $\mathbb E$ is the expectation operator induced by $\PP$. 
\end{lemma}
This lemma appeared as Lemma 15.1 in the book of \cite{lattimore2020bandit}, where the reader can also find the proof.

\begin{lemma}[Lemma 20 in \cite{jaksch2010near}]\label{lemma:supp1}
Suppose $0\leq q\leq 1/2$ and $\epsilon\leq 1-2q$, then
\begin{equation*}
    q\log\Big(\frac{q}{q+\epsilon}\Big)+(1-q)\log\Big(\frac{1-q}{1-q-\epsilon}\Big)\leq \frac{2\epsilon^2}{q}\,.
\end{equation*}
\end{lemma}

\begin{lemma}[Pinsker's inequality]\label{lemma:pinsker}
For measures $P$ and $Q$ on the same probability space $(\Omega, \cF)$, we have 
\begin{equation*}
    \delta(P, Q) = \sup_{A\in\cF}(P(A)-Q(A))\leq \sqrt{\frac{1}{2}\KL(P, Q)}.
\end{equation*}
\end{lemma}

\bibliographystyle{plainnat}
{\small
\bibliography{ref}

\begin{thebibliography}{48}
\providecommand{\natexlab}[1]{#1}
\providecommand{\url}[1]{\texttt{#1}}
\expandafter\ifx\csname urlstyle\endcsname\relax
  \providecommand{\doi}[1]{doi: #1}\else
  \providecommand{\doi}{doi: \begingroup \urlstyle{rm}\Url}\fi

\bibitem[Abbasi-Yadkori et~al.(2012)Abbasi-Yadkori, Pal, and
  Szepesvari]{abbasi2012online}
Yasin Abbasi-Yadkori, David Pal, and Csaba Szepesvari.
\newblock Online-to-confidence-set conversions and application to sparse
  stochastic bandits.
\newblock In \emph{Artificial Intelligence and Statistics}, pages 1--9, 2012.

\bibitem[Abbasi-Yadkori et~al.(2019{\natexlab{a}})Abbasi-Yadkori, Bartlett,
  Bhatia, Lazic, Szepesvari, and Weisz]{abbasi2019politex}
Yasin Abbasi-Yadkori, Peter Bartlett, Kush Bhatia, Nevena Lazic, Csaba
  Szepesvari, and Gell{\'e}rt Weisz.
\newblock Politex: Regret bounds for policy iteration using expert prediction.
\newblock In \emph{International Conference on Machine Learning}, pages
  3692--3702, 2019{\natexlab{a}}.

\bibitem[Abbasi-Yadkori et~al.(2019{\natexlab{b}})Abbasi-Yadkori, Lazic,
  Szepesvari, and Weisz]{abbasi2019exploration}
Yasin Abbasi-Yadkori, Nevena Lazic, Csaba Szepesvari, and Gellert Weisz.
\newblock Exploration-enhanced politex.
\newblock \emph{arXiv preprint arXiv:1908.10479}, 2019{\natexlab{b}}.

\bibitem[Agarwal et~al.(2020{\natexlab{a}})Agarwal, Kakade, Krishnamurthy, and
  Sun]{AgKaKrSu20}
Alekh Agarwal, Sham Kakade, Akshay Krishnamurthy, and Wen Sun.
\newblock {FLAMBE}: Structural complexity and representation learning of low
  rank {MDP}s.
\newblock \emph{arXiv preprint arXiv:2006.10814v2}, 2020{\natexlab{a}}.

\bibitem[Agarwal et~al.(2020{\natexlab{b}})Agarwal, Kakade, Lee, and
  Mahajan]{agarwal2020theory}
Alekh Agarwal, Sham~M. Kakade, Jason~D. Lee, and Gaurav Mahajan.
\newblock On the theory of policy gradient methods: Optimality, approximation,
  and distribution shift, 2020{\natexlab{b}}.

\bibitem[Antos et~al.(2008)Antos, Szepesv{\'a}ri, and Munos]{antos2008fitted}
Andr{\'a}s Antos, Csaba Szepesv{\'a}ri, and R{\'e}mi Munos.
\newblock Fitted {Q}-iteration in continuous action-space {MDP}s.
\newblock In \emph{Advances in neural information processing systems}, pages
  9--16, 2008.

\bibitem[Azar et~al.(2017)Azar, Osband, and Munos]{azar2017minimax}
Mohammad~Gheshlaghi Azar, Ian Osband, and R{\'e}mi Munos.
\newblock Minimax regret bounds for reinforcement learning.
\newblock In \emph{International Conference on Machine Learning}, pages
  263--272, 2017.

\bibitem[Bastani and Bayati(2020)]{bastani2020online}
Hamsa Bastani and Mohsen Bayati.
\newblock Online decision making with high-dimensional covariates.
\newblock \emph{Operations Research}, 68\penalty0 (1):\penalty0 276--294, 2020.

\bibitem[Bellman et~al.(1963)Bellman, Kalaba, and Kotkin]{BeKaKo63}
I.~R. Bellman, R.~Kalaba, and B.~Kotkin.
\newblock Polynomial approximation -- a new computational technique in dynamic
  programming.
\newblock \emph{Math. Comp.}, 17\penalty0 (8):\penalty0 155--161, 1963.

\bibitem[Bertsekas(1995)]{bertsekas1995dynamic}
Dimitri~P. Bertsekas.
\newblock \emph{Dynamic programming and optimal control}, volume~1.
\newblock Athena Scientific, 1995.

\bibitem[Bickel et~al.(2009)Bickel, Ritov, Tsybakov,
  et~al.]{bickel2009simultaneous}
Peter~J Bickel, Ya'acov Ritov, Alexandre~B Tsybakov, et~al.
\newblock Simultaneous analysis of {L}asso and {D}antzig selector.
\newblock \emph{The Annals of Statistics}, 37\penalty0 (4):\penalty0
  1705--1732, 2009.

\bibitem[B{\"u}hlmann and Van De~Geer(2011)]{buhlmann2011statistics}
Peter B{\"u}hlmann and Sara Van De~Geer.
\newblock \emph{Statistics for high-dimensional data: methods, theory and
  applications}.
\newblock Springer Science \& Business Media, 2011.

\bibitem[Cai et~al.(2019)Cai, Yang, Jin, and Wang]{cai2019provably}
Qi~Cai, Zhuoran Yang, Chi Jin, and Zhaoran Wang.
\newblock Provably efficient exploration in policy optimization.
\newblock \emph{arXiv preprint arXiv:1912.05830}, 2019.

\bibitem[Duan and Wang(2020)]{duan2020minimax}
Yaqi Duan and Mengdi Wang.
\newblock Minimax-optimal off-policy evaluation with linear function
  approximation.
\newblock \emph{Internation Conference on Machine Learning}, 2020.

\bibitem[Ernst et~al.(2005)Ernst, Geurts, and Wehenkel]{ernst2005tree}
Damien Ernst, Pierre Geurts, and Louis Wehenkel.
\newblock Tree-based batch mode reinforcement learning.
\newblock \emph{Journal of Machine Learning Research}, 6\penalty0
  (Apr):\penalty0 503--556, 2005.

\bibitem[Geist and Scherrer(2011)]{geist2011}
Matthieu Geist and Bruno Scherrer.
\newblock $\ell^1$-penalized projected {B}ellman residual.
\newblock In \emph{European Workshop on Reinforcement Learning}, pages 89--101.
  Springer, 2011.

\bibitem[Geist et~al.(2012)Geist, Scherrer, Lazaric, and
  Ghavamzadeh]{geist2012dantzig}
Matthieu Geist, Bruno Scherrer, Alessandro Lazaric, and Mohammad Ghavamzadeh.
\newblock A {D}antzig selector approach to temporal difference learning.
\newblock In \emph{Proceedings of the 29th International Coference on
  International Conference on Machine Learning}, pages 347--354, 2012.

\bibitem[Ghavamzadeh et~al.(2011)Ghavamzadeh, Lazaric, Munos, and
  Hoffman]{ghavamzadeh2011finite}
Mohammad Ghavamzadeh, Alessandro Lazaric, R{\'e}mi Munos, and Matthew Hoffman.
\newblock Finite-sample analysis of {L}asso-{TD}.
\newblock In \emph{Proceedings of the 28th International Conference on
  International Conference on Machine Learning}, pages 1177--1184, 2011.

\bibitem[Hao et~al.(2020{\natexlab{a}})Hao, Duan, Lattimore, Szepesv{\'a}ri,
  and Wang]{hao2020sparse}
Botao Hao, Yaqi Duan, Tor Lattimore, Csaba Szepesv{\'a}ri, and Mengdi Wang.
\newblock Sparse feature selection makes batch reinforcement learning more
  sample efficient.
\newblock \emph{arXiv preprint arXiv:2011.04019}, 2020{\natexlab{a}}.

\bibitem[Hao et~al.(2020{\natexlab{b}})Hao, Lattimore, and Wang]{hao2020high}
Botao Hao, Tor Lattimore, and Mengdi Wang.
\newblock High-dimensional sparse linear bandits.
\newblock \emph{Advances in Neural Information Processing Systems}, 33,
  2020{\natexlab{b}}.

\bibitem[Hoffman et~al.(2011)Hoffman, Lazaric, Ghavamzadeh, and
  Munos]{hoffman2011regularized}
Matthew~W Hoffman, Alessandro Lazaric, Mohammad Ghavamzadeh, and R{\'e}mi
  Munos.
\newblock Regularized least squares temporal difference learning with nested
  $\ell^2$ and $\ell^1$ penalization.
\newblock In \emph{European Workshop on Reinforcement Learning}, pages
  102--114. Springer, 2011.

\bibitem[Ibrahimi et~al.(2012)Ibrahimi, Javanmard, and
  Roy]{ibrahimi2012efficient}
Morteza Ibrahimi, Adel Javanmard, and Benjamin~V Roy.
\newblock Efficient reinforcement learning for high dimensional linear
  quadratic systems.
\newblock In \emph{Advances in Neural Information Processing Systems}, pages
  2636--2644, 2012.

\bibitem[Jaksch et~al.(2010)Jaksch, Ortner, and Auer]{jaksch2010near}
Thomas Jaksch, Ronald Ortner, and Peter Auer.
\newblock Near-optimal regret bounds for reinforcement learning.
\newblock \emph{Journal of Machine Learning Research}, 11:\penalty0 1563--1600,
  2010.

\bibitem[Jiang et~al.(2017)Jiang, Krishnamurthy, Agarwal, Langford, and
  Schapire]{JiKrAgLaSch17}
Nan Jiang, Akshay Krishnamurthy, Alekh Agarwal, John Langford, and Robert~E.
  Schapire.
\newblock Contextual decision processes with low {B}ellman rank are
  {PAC}-learnable.
\newblock In \emph{International Conference on Machine Learning}. JMLR.org,
  2017.

\bibitem[Jin et~al.(2018)Jin, Allen-Zhu, Bubeck, and Jordan]{jin2018q}
Chi Jin, Zeyuan Allen-Zhu, Sebastien Bubeck, and Michael~I Jordan.
\newblock Is q-learning provably efficient?
\newblock In \emph{Advances in Neural Information Processing Systems}, pages
  4863--4873, 2018.

\bibitem[Jin et~al.(2019)Jin, Yang, Wang, and Jordan]{jin2019provably}
Chi Jin, Zhuoran Yang, Zhaoran Wang, and Michael~I Jordan.
\newblock Provably efficient reinforcement learning with linear function
  approximation.
\newblock \emph{arXiv preprint arXiv:1907.05388}, 2019.

\bibitem[Kim and Paik(2019)]{kim2019doubly}
Gi-Soo Kim and Myunghee~Cho Paik.
\newblock Doubly-robust lasso bandit.
\newblock In \emph{Advances in Neural Information Processing Systems}, pages
  5869--5879, 2019.

\bibitem[Kolter and Ng(2009)]{kolter2009regularization}
J~Zico Kolter and Andrew~Y Ng.
\newblock Regularization and feature selection in least-squares temporal
  difference learning.
\newblock In \emph{Proceedings of the 26th annual international conference on
  machine learning}, pages 521--528, 2009.

\bibitem[Lattimore and Szepesv{\'a}ri(2020)]{lattimore2020bandit}
Tor Lattimore and Csaba Szepesv{\'a}ri.
\newblock \emph{Bandit algorithms}.
\newblock Cambridge University Press, 2020.

\bibitem[Lazic et~al.(2020)Lazic, Yin, Farajtabar, Levine, Gorur, Harris, and
  Schuurmans]{lazic2020maximum}
Nevena Lazic, Dong Yin, Mehrdad Farajtabar, Nir Levine, Dilan Gorur, Chris
  Harris, and Dale Schuurmans.
\newblock A maximum-entropy approach to off-policy evaluation in average-reward
  mdps.
\newblock \emph{Conference on Neural Information Processing Systems}, 2020.

\bibitem[Liu et~al.(2012)Liu, Mahadevan, and Liu]{liu2012regularized}
Bo~Liu, Sridhar Mahadevan, and Ji~Liu.
\newblock Regularized off-policy {TD}-learning.
\newblock In \emph{Advances in Neural Information Processing Systems}, pages
  836--844, 2012.

\bibitem[Painter-Wakefield and Parr(2012)]{painter2012greedy}
Christopher Painter-Wakefield and Ronald Parr.
\newblock Greedy algorithms for sparse reinforcement learning.
\newblock In \emph{Proceedings of the 29th International Coference on
  International Conference on Machine Learning}, pages 867--874, 2012.

\bibitem[Puterman(2014)]{puterman2014Markov}
Martin~L Puterman.
\newblock \emph{{M}arkov Decision Processes.: Discrete Stochastic Dynamic
  Programming}.
\newblock John Wiley \& Sons, 2014.

\bibitem[Ren and Zhou(2020)]{ren2020dynamic}
Zhimei Ren and Zhengyuan Zhou.
\newblock Dynamic batch learning in high-dimensional sparse linear contextual
  bandits.
\newblock \emph{arXiv preprint arXiv:2008.11918}, 2020.

\bibitem[Schweitzer and Seidmann(1985)]{SchSe85}
Paul~J Schweitzer and Abraham Seidmann.
\newblock Generalized polynomial approximations in {M}arkovian decision
  processes.
\newblock \emph{Journal of Mathematical Analysis and Applications},
  110\penalty0 (2):\penalty0 568--582, 1985.

\bibitem[Sun et~al.(2019)Sun, Jiang, Krishnamurthy, Agarwal, and
  Langford]{sun2019model}
Wen Sun, Nan Jiang, Akshay Krishnamurthy, Alekh Agarwal, and John Langford.
\newblock Model-based rl in contextual decision processes: Pac bounds and
  exponential improvements over model-free approaches.
\newblock In \emph{Conference on Learning Theory}, pages 2898--2933, 2019.

\bibitem[Szepesv{\'a}ri(2010)]{sze10}
Csaba Szepesv{\'a}ri.
\newblock \emph{Algorithms for Reinforcement Learning}.
\newblock Morgan and Claypool, 2010.

\bibitem[Tibshirani(1996)]{tibshirani1996regression}
Robert Tibshirani.
\newblock Regression shrinkage and selection via the {L}asso.
\newblock \emph{Journal of the Royal Statistical Society: Series B
  (Methodological)}, 58\penalty0 (1):\penalty0 267--288, 1996.

\bibitem[Tsybakov(2008)]{Tsybakov:2008:INE:1522486}
Alexandre~B. Tsybakov.
\newblock \emph{Introduction to Nonparametric Estimation}.
\newblock Springer Publishing Company, Incorporated, 1st edition, 2008.
\newblock ISBN 0387790519, 9780387790510.

\bibitem[Van De~Geer et~al.(2009)Van De~Geer, B{\"u}hlmann,
  et~al.]{van2009conditions}
Sara~A Van De~Geer, Peter B{\"u}hlmann, et~al.
\newblock On the conditions used to prove oracle results for the lasso.
\newblock \emph{Electronic Journal of Statistics}, 3:\penalty0 1360--1392,
  2009.

\bibitem[Vershynin(2010)]{vershynin2010introduction}
Roman Vershynin.
\newblock Introduction to the non-asymptotic analysis of random matrices.
\newblock \emph{arXiv preprint arXiv:1011.3027}, 2010.

\bibitem[Wainwright(2019)]{Wa19}
Martin~J Wainwright.
\newblock \emph{High-Dimensional Statistics: A Non-Asymptotic Viewpoint}.
\newblock Cambridge University Press, 2019.

\bibitem[Wang et~al.(2018)Wang, Wei, and Yao]{wang2018minimax}
Xue Wang, Mingcheng Wei, and Tao Yao.
\newblock Minimax concave penalized multi-armed bandit model with
  high-dimensional covariates.
\newblock In \emph{International Conference on Machine Learning}, pages
  5200--5208, 2018.

\bibitem[Wang et~al.(2020)Wang, Chen, Fang, Wang, and Li]{wang2020nearly}
Yining Wang, Yi~Chen, Ethan~X Fang, Zhaoran Wang, and Runze Li.
\newblock Nearly dimension-independent sparse linear bandit over small action
  spaces via best subset selection.
\newblock \emph{arXiv preprint arXiv:2009.02003}, 2020.

\bibitem[Yang and Wang(2019)]{yang2019sample}
Lin Yang and Mengdi Wang.
\newblock Sample-optimal parametric {Q}-learning using linearly additive
  features.
\newblock In \emph{International Conference on Machine Learning}, pages
  6995--7004, 2019.

\bibitem[Yang and Wang(2020)]{yang2019reinforcement}
Lin~F Yang and Mengdi Wang.
\newblock Reinforcement leaning in feature space: Matrix bandit, kernels, and
  regret bound.
\newblock \emph{International Conference on Machine Learning}, 2020.

\bibitem[Zanette et~al.(2020)Zanette, Brandfonbrener, Brunskill, Pirotta, and
  Lazaric]{zanette2020frequentist}
Andrea Zanette, David Brandfonbrener, Emma Brunskill, Matteo Pirotta, and
  Alessandro Lazaric.
\newblock Frequentist regret bounds for randomized least-squares value
  iteration.
\newblock In \emph{International Conference on Artificial Intelligence and
  Statistics}, pages 1954--1964, 2020.

\bibitem[Zhou et~al.(2020)Zhou, He, and Gu]{zhou2020provably}
Dongruo Zhou, Jiafan He, and Quanquan Gu.
\newblock Provably efficient reinforcement learning for discounted mdps with
  feature mapping.
\newblock \emph{arXiv preprint arXiv:2006.13165}, 2020.

\end{thebibliography}
}

\end{document}